%% file: neurips_2026.tex
\documentclass{article}

\PassOptionsToPackage{authoryear,round}{natbib} 

\usepackage[final,main]{neurips_2026}

\let\cite\citep

\usepackage[utf8]{inputenc}
\usepackage[T1]{fontenc}
\usepackage{booktabs}
\usepackage{amsmath}
\usepackage{amsfonts}
\usepackage{amsthm}
\usepackage{algorithm}
\usepackage{algpseudocode}
\usepackage{graphicx}
\usepackage{nicefrac}
\usepackage{microtype}
\usepackage{xcolor}
\usepackage{tikz}
\usepackage{epstopdf}
\usepackage{xspace}
\usepackage{placeins}
\usepackage{siunitx}
\usetikzlibrary{arrows.meta,calc,decorations.pathreplacing,fit,positioning}
\usepackage{url}
\usepackage{hyperref}

\hypersetup{
  hidelinks,
  pdftitle={Generative Actor-Critic with Soft Bridge Policies},
  pdfauthor={Ke He, Le He, Shunpu Tang, Yafei Wang, Lisheng Fan}
}

\definecolor{softgacblue}{HTML}{1F77B4}
\definecolor{softgaclightblue}{HTML}{DCEBFA}
\definecolor{softgacorange}{HTML}{D95F02}
\definecolor{softgaclightorange}{HTML}{FCE1C3}
\definecolor{softgacgray}{HTML}{666666}
\definecolor{softgaclightgray}{HTML}{F3F4F6}

\newtheorem{proposition}{Proposition}
\newtheorem{theorem}{Theorem}
\newtheorem{lemma}{Lemma}
\newtheorem{corollary}{Corollary}

\newcommand{\myalgo}{\texttt{SoftGAC}\xspace}

\title{Generative Actor-Critic with Soft Bridge Policies}

\author{%
Ke He \\
Guangzhou University \\
\texttt{ke.he@gzhu.edu.cn}
\And
Le He\thanks{Corresponding author: Le He (\texttt{hele@e.gzhu.edu.cn}).} \\
Guangzhou University \\
\texttt{hele@e.gzhu.edu.cn}
\And
Shunpu Tang \\
Zhejiang University \\
\texttt{tangshunpu@zju.edu.cn}
\AND
Yafei Wang \\
Southeast University \\
\texttt{wangyf@seu.edu.cn}
\And
Lisheng Fan \\
Guangzhou University \\
\texttt{lsfan@gzhu.edu.cn}
}

\begin{document}

\maketitle


\begin{abstract}
Expressive generative policies such as diffusion and flow models are appealing for maximum-entropy (MaxEnt) online reinforcement learning because of their ability to model multimodal and highly non-Gaussian action distributions. However, training effective soft generative policies faces two obstacles that often arise together. First, marginal action densities are often unavailable, so existing methods typically rely on entropy bounds, heuristic proxies or approximations. Second, iterative shared-parameter samplers raise inference cost and require backpropagation through time over repeated network evaluations, increasing memory cost and potentially destabilizing policy optimization. These obstacles motivate us to seek a generative actor-critic method that exposes a tractable soft objective while generating each action in just one network evaluation. To this end, we propose \emph{soft generative actor-critic} (\myalgo), whose actor is a single stochastic feedforward network with internal Gaussian bridge layers. This explicit bridge structure allows us to lift the soft objective to path space, yielding an analytical path-KL regularizer against a high-entropy reference process for both actor improvement and the soft Bellman bootstrap, without estimating the endpoint entropy. Crucially, the actor generates each action in only one network evaluation under a total parameter budget comparable to strong baselines, substantially reducing inference cost relative to the existing iterative diffusion and flow policies. It also avoids long unrolled backpropagation through time. Across challenging continuous-control benchmarks, our method achieves competitive or superior returns compared with strong diffusion and flow policies, demonstrating a clear advantage in the compute-return tradeoff.
\end{abstract}

\begin{center}
  \textbf{Code:} \url{https://github.com/skypitcher/soft_gac_jax}
\end{center}

\section{Introduction}

Maximum-entropy (MaxEnt) online reinforcement learning (RL) has become a central approach for off-policy continuous-control because it combines value improvement with explicit stochasticity~\citep{levine2018reinforcement, oh2025discovering, haarnoja2018soft, SAC18}. In soft actor-critic (SAC)~\citep{SAC18}, this stochasticity is not only an exploration device. It is part of the optimization objective, where the actor is encouraged to choose high-value actions while maintaining high entropy~\cite{haarnoja2018soft}. This principle is simple and effective when the policy has a tractable density, but the usual Gaussian policy class can be too restrictive for complex control problems that may require multimodal or highly non-Gaussian action distributions~\cite{nauman2024bigger, NormalizingFlowRL24, DDiffPG24, FQL25}.

Expressive generative policies offer a natural way to expand the actor class~\cite{Diffusion20, ScoreDiff21, FlowMatching22,FPO25,DPPO25}. Diffusion, score-based and flow-style policies can represent rich action distributions and have recently become attractive for reinforcement learning~\citep{QSM24,QVPO24,D3P25,FlowRL25,ReinFlow25}. Their appeal, however, creates a tension with MaxEnt RL. The soft actor update is an endpoint-density objective that requires policy entropy or an equivalent density-based regularizer. For many generative actors, the marginal action density is unavailable or expensive to evaluate because the action is obtained after marginalizing over latent variables or generation paths. Existing methods therefore often introduce entropy bounds~\cite{DIME25}, entropy estimators and approximations~\cite{ACER24, QVPO24}, noise-augmented path likelihoods~\cite{SACFlow26} or Wasserstein-geometric proxies~\cite{FLAC26,WPPG26}, which make the connection to MaxEnt RL indirect.

A second difficulty comes from how many generative policies spend computation. Diffusion or flow-style samplers with many denoising or refinement steps usually reuse the same network across time~\cite{DPPO25,DIME25,QSM24,QVPO24}. This raises inference cost at deployment and it also makes the number of function evaluations (NFE) a training-time cost, because the actor update differentiates $Q(s,a)$ through the sampled action and every refinement step that produced it~\cite{OFQL26}. Increasing NFE therefore increases BPTT depth, activation memory and actor-update wall-clock time~\cite{fransone}. It can also lengthen gradient paths through repeated uses of shared parameters, which may destabilize actor optimization. As a result, high-NFE policies may achieve stronger performance, but the added optimization burden often limits their marginal gains~\cite{QVPO24}. Conversely, reducing NFE lowers both training and inference cost, but can cause a sharp performance drop~\cite{FQL25,DIME25}.

These two obstacles motivate us to seek a different design principle. Specifically, we seek a generative actor-critic design that retains the expressiveness of stochastic latent generation while still exposing a tractable soft objective and generating each action with only one network evaluation. Our goal is not to trade policy performance for cheaper action generation, nor to make soft regularization rely on entropy bounds, heuristic proxies or endpoint-entropy estimates. Instead, we ask whether careful actor-critic design can make an expressive generative policy both faster and stronger while still providing an analytical soft regularizer. We aim to match or exceed the returns of strong iterative diffusion and flow policies using only one network evaluation per action under a comparable total actor parameter budget. To achieve this goal, our key idea is to model the internal stochastic computation of a single feedforward actor as an explicit latent path, rather than treating the actor only through its terminal action distribution. This path-law view replaces endpoint entropy with path-wise Kullback–Leibler (KL) regularization relative to a high-entropy reference process. We show that, with unrestricted path laws and a uniform reference action marginal, the resulting objective recovers the same optimal action distribution as standard MaxEnt policy improvement while additionally regularizing how actions are generated.

We instantiate this idea as \emph{soft generative actor-critic} (\myalgo), an off-policy actor-critic method with soft bridge policies, as illustrated in Figure~\ref{fig:softgac_overview}. Its actor is a single stochastic feedforward network whose lightweight internal Gaussian transitions form a bridge in pre-tanh latent space from a fixed base latent to a terminal action latent. These explicit transitions make the local KL terms analytically tractable, allowing us to evaluate the finite-step regularizer as transition control energy along each sampled path. The same forward pass thus generates an action and computes its control energy, which the actor update directly trades off against critic value.

Our contributions are three-fold. \textbf{(i)} We formulate a path-space soft objective for generative actors, show that it contains the endpoint KL regularizer used by MaxEnt RL as a marginal component and characterize both its unrestricted endpoint-equivalent optimum and the effect of using a practical fixed actor base. 
\textbf{(ii)} We design a soft bridge policy that uses a short sequence of local Gaussian transitions in pre-tanh latent space, making the path-wise regularizer exactly computable as finite-step control energy and enabling a direct actor update with action generation by a single feedforward pass under a comparable actor parameter budget. 
\textbf{(iii)} We demonstrate on challenging continuous-control benchmarks that \myalgo improves the compute-return tradeoff by pairing strong returns with low one-pass action-generation cost and a parameter budget comparable to strong actor baselines.

\begin{figure}[t]
  \centering
  \includegraphics[width=\linewidth]{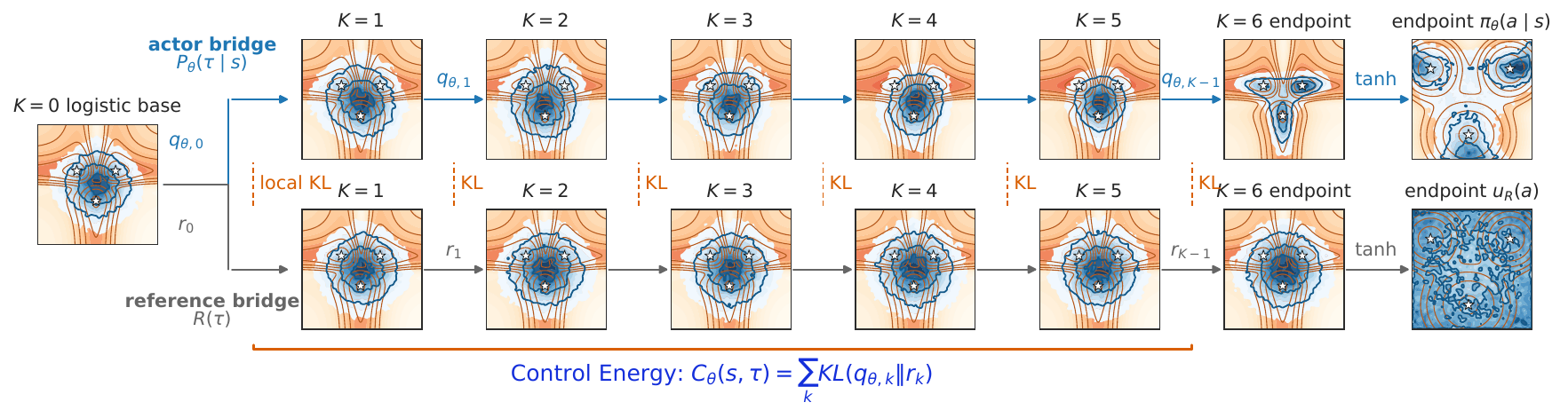}
  \caption{Overview of the proposed soft bridge policy. Local KL terms compare actor and reference transitions and sum to the sampled control energy. The terminal latent is mapped through $\tanh$ and optimized by the critic. Appendix~\ref{app:controlled_2d} provides a 2D bridge visualization.}
  \label{fig:softgac_overview}
\end{figure}

\section{Preliminaries}

\paragraph{Maximum-entropy reinforcement learning.}
We consider a discounted Markov decision process with state space $\mathcal S$, bounded continuous action space $\mathcal A \subset \mathbb R^{d_a}$, reward $r(s,a)$, transition kernel $p(s' \mid s,a)$, discount factor $\gamma \in (0,1)$ and policy $\pi(a \mid s)$. Maximum-entropy reinforcement learning augments the expected return with an entropy bonus,
\begin{equation}
  J(\pi)
  =
  \mathbb E_\pi
  \left[
    \sum_{t=0}^{\infty}
    \gamma^t
    \left(
      r(s_t,a_t)
      + \alpha \mathcal H(\pi(\cdot \mid s_t))
    \right)
  \right],
\end{equation}
where $\alpha>0$ controls the strength of the soft regularizer. In an off-policy actor-critic algorithm, a critic estimates a soft value landscape and the actor is improved by increasing value while preserving stochasticity. When $\mathcal A$ is bounded and $u$ denotes the uniform action law, the standard per-state soft actor improvement step instantiated by SAC~\cite{SAC18} can be written as
\begin{equation}
  \mathcal J_{\mathrm{SAC}}(\pi \mid s)
  =
  \mathbb E_{a\sim\pi(\cdot\mid s)}[Q(s,a)]
  -
  \alpha D_{\mathrm{KL}}(\pi(\cdot\mid s)\,\|\,u)
  + \mathrm{const}.
  \label{eq:sac_endpoint_kl}
\end{equation}
This form will be useful below because it separates the value term from a regularizer that keeps the endpoint action distribution close to a high-entropy reference. Explicit-density policies such as standard SAC can evaluate this objective because the policy density is available. The difficulty begins when the policy is a generative sampler whose endpoint density is not directly exposed.

\paragraph{Generative actors as path laws.}
Many recent generative policies naturally fit a common path-law view. A diffusion actor~\cite{Diffusion20,DDiffPG24,DIME25} starts from base noise and samples a reverse denoising chain,
\begin{equation}
  y_{k-1}\sim p_\theta(y_{k-1}\mid y_k,s),
  \qquad
  y_K\sim\mathcal N(0,I),
\end{equation}
while a flow or flow-matching actor~\cite{FlowMatching22,FPO25,ReinFlow25} evolves base noise through a learned velocity field,
\begin{equation}
  \frac{d y_t}{d t}=v_\theta(y_t,t,s),
  \qquad
  y_0\sim p_0.
\end{equation}
Beyond multi-step implicit generators, one-shot implicit actors form a degenerate short-path case, where $z_0\sim p_0$ is mapped by a neural generator to a terminal latent $z_1=f_\theta(z_0,s)$~\cite{WPPG26}. These actors all generate an action by following a latent path from base noise to a terminal latent state. For the rest of the paper, we use the forward ordering $\tau=(z_0,z_1,\ldots,z_K)$, where $z_0$ denotes the base latent and $z_K$ denotes the terminal action latent. The actor samples this path from a path law $P_\theta(d\tau\mid s)$ and produces an action through a terminal map $a=T(z_K)$. In this paper, $T$ maps the terminal latent state to a bounded action through a tanh transform, but we only need the induced endpoint law $\pi_\theta(\cdot\mid s) = (T\circ z_K)_{\#}P_\theta(\cdot\mid s)$, where $\#$ denotes the pushforward of a probability law. Equivalently, this is the distribution of $T(z_K)$ when $\tau\sim P_\theta(\cdot\mid s)$. Such an actor can be easy to sample from while still having an intractable marginal action density. The density of $\pi_\theta(a\mid s)$ may require integrating over all latent paths that terminate at the same action, which is generally unavailable for implicit generators and costly for multi-step diffusion or flow samplers.

\section{Maximum-Entropy Reinforcement Learning in Path Space}

\subsection{Path-Space Objective}

When the terminal action density is unavailable, a natural approach is to lift MaxEnt RL from the terminal action distribution to the full generation path. Since a generative actor samples a latent path $\tau$ before producing a terminal action, we define the soft objective over path laws. Let $R(d\tau)$ be a high-entropy reference path law in the same latent space and let $u_R$ be its terminal action law induced by $T(z_K)$. For a fixed state $s$ and critic $Q$, we consider
\begin{equation}
  \mathcal J_{\mathrm{path}}(P_\theta \mid s)
  =
  \mathbb E_{\tau\sim P_\theta(\cdot\mid s)}
  \left[
    Q(s,T(z_K))
  \right]
  -
  \alpha
  D_{\mathrm{KL}}
  \left(
    P_\theta(\cdot\mid s)
    \,\|\,R
  \right).
  \label{eq:path_soft_objective}
\end{equation}
The value term still evaluates only the terminal action. The regularizer, however, now compares the full generation path to a high-entropy reference process. This path-space objective is useful only if it remains connected to the endpoint MaxEnt objective. We show this connection by decomposing the path KL into a terminal-action term and a conditional path term.

\begin{proposition}[Path-space lift of the endpoint KL regularizer]
Let $a=T(z_K)$ be the terminal action. Suppose the actor path law and the reference path law admit disintegrations
\begin{equation}
  P_\theta(d\tau\mid s)
  =
  \pi_\theta(da\mid s)\,P_\theta(d\tau\mid s,a),
  \qquad
  R(d\tau)
  =
  u_R(da)\,R(d\tau\mid a).
\end{equation}
Then
\begin{equation}
  D_{\mathrm{KL}}
  \left(
    P_\theta(\cdot\mid s)
    \,\|\,R
  \right)
  =
  D_{\mathrm{KL}}
  \left(
    \pi_\theta(\cdot\mid s)
    \,\|\,u_R
  \right)
  +
  \mathbb E_{a\sim\pi_\theta}
  \left[
    D_{\mathrm{KL}}
    \left(
      P_\theta(\cdot\mid s,a)
      \,\|\,R(\cdot\mid a)
    \right)
  \right].
  \label{eq:path_kl_decomposition}
\end{equation}
\end{proposition}

The proof is a direct chain-rule decomposition of relative entropy and appears in Appendix~\ref{app:proof_path_lift}. This decomposition shows that the path KL contains the endpoint KL regularizer as its marginal component. When $u_R$ is uniform over the bounded action space, the first term in Eq.~\eqref{eq:path_kl_decomposition} is exactly the uniform-prior endpoint regularizer in Eq.~\eqref{eq:sac_endpoint_kl} up to a constant. The second term is the additional structure introduced by the path-space lift. It regularizes how the actor generates an action, not only which terminal action distribution it induces. Thus the path objective is stricter than endpoint entropy while still preserving the endpoint regularizer of MaxEnt RL.

\subsection{Unrestricted Soft-Optimal Bridge}

We next ask whether this stricter objective changes the ideal soft-optimal endpoint action distribution. If we optimize over all path laws that are absolutely continuous with respect to the reference, the answer is no in the ideal uniform-reference case.

\begin{theorem}[Unrestricted soft-optimal bridge]
For fixed $s$, $Q$ and $\alpha>0$, consider the unrestricted optimization of Eq.~\eqref{eq:path_soft_objective} over path laws $P$ with $P\ll R$. Assume
\begin{equation}
  Z_Q(s)
  =
  \mathbb E_{\tau\sim R}
  \left[
    \exp
    \left(
      Q(s,T(z_K))/\alpha
    \right)
  \right]
  <
  \infty.
\end{equation}
The unique optimum is the value-tilted reference path law
\begin{equation}
  \frac{dP_Q^\star}{dR}(\tau\mid s)
  =
  \frac{
    \exp
    \left(
      Q(s,T(z_K))/\alpha
    \right)
  }{
    Z_Q(s)
  }.
  \label{eq:unrestricted_tilted_bridge}
\end{equation}
Its terminal action law is
\begin{equation}
  \pi_Q^\star(a\mid s)
  =
  \frac{
    u_R(a)
    \exp
    \left(
      Q(s,a)/\alpha
    \right)
  }{
    \int_{\mathcal A}
    u_R(b)
    \exp
    \left(
      Q(s,b)/\alpha
    \right)
    db
  }.
  \label{eq:unrestricted_endpoint}
\end{equation}
\end{theorem}

The proof follows from the Gibbs variational identity and appears in Appendix~\ref{app:proof_unrestricted_bridge}. Since the value tilt in Eq.~\eqref{eq:unrestricted_tilted_bridge} depends only on the terminal action, the optimal conditional path law given that action remains the reference conditional bridge $R(d\tau\mid a)$. Value changes the endpoint marginal, while the reference controls how each endpoint is reached. When $u_R$ is uniform, Eq.~\eqref{eq:unrestricted_endpoint} recovers the usual Boltzmann endpoint policy proportional to $\exp(Q(s,a)/\alpha)$. The unrestricted path-space objective is therefore \emph{endpoint-equivalent} to the MaxEnt actor update in the ideal reference limit, but it is more selective about the generation path.

\subsection{Fixed-Base Bridge Choice}

The unrestricted optimum gives a clean conceptual target, but every implementable generative actor also needs to specify how its latent path starts. A diffusion sampler, flow sampler or one-shot generator all begin from some base law. This base law is a practical modeling choice. If the actor base differs from the reference base, the path KL gains an initial-law term before the remaining conditional path is compared to the reference. For a fixed actor base, this initial term is constant with respect to the conditional path law after $z_0$. Thus the base choice does not change the variational problem solved after conditioning on the base sample. In finite implementations, it mainly acts as an inductive bias by shaping the initial latents from which the rest of the generative path is drawn. To characterize this effect, let the reference path law disintegrate as
\begin{equation}
  R(d\tau)
  =
  r_0(dz_0)\,
  R(d\tau_{1:K}\mid z_0),
\end{equation}
and fix an actor base law $p_0$ with $p_0\ll r_0$. We optimize the same path objective as Eq.~\eqref{eq:path_soft_objective}, but restrict the actor path law to satisfy $P_0=p_0$. The value tilt still points toward the unrestricted soft optimum in Theorem~1, while the actor can only search within the selected fixed-base family. The best attainable solution is therefore the constrained optimum below.

\begin{theorem}[Fixed-base bridge optimum]
\label{thm:fixed_base_optimum}
Define
\begin{equation}
  Z_Q(s,z_0)
  =
  \mathbb E_{\tau\sim R(\cdot\mid z_0)}
  \left[
    \exp
    \left(
      Q(s,T(z_K))/\alpha
    \right)
  \right],
  \qquad
  \bar Z_Q(s)
  =
  \mathbb E_{z_0\sim r_0}
  \left[
    Z_Q(s,z_0)
  \right].
\end{equation}
Among all path laws with initial marginal $p_0$, the unique optimum is
\begin{equation}
  P_{Q,p_0}^\star(d\tau\mid s)
  =
  p_0(dz_0)\,
  \frac{
    \exp
    \left(
      Q(s,T(z_K))/\alpha
    \right)
  }{
    Z_Q(s,z_0)
  }
  R(d\tau_{1:K}\mid z_0).
  \label{eq:shared_base_optimum}
\end{equation}
Its objective value is
\begin{equation}
  \alpha
  \mathbb E_{z_0\sim p_0}
  \left[
    \log Z_Q(s,z_0)
  \right]
  -
  \alpha
  D_{\mathrm{KL}}
  \left(
    p_0
    \,\|\,
    r_0
  \right).
  \label{eq:fixed_base_objective}
\end{equation}
\end{theorem}

The proof applies the Gibbs variational identity conditionally on the fixed base latent and appears in Appendix~\ref{app:proof_fixed_base_optimum}. The initial KL term is independent of the conditional optimizer after $z_0$. Thus using a different fixed base is not a heuristic break from the objective. It changes the full path objective by a constant and changes the finite actor's inductive bias through the initial samples. When the fixed actor base matches the value-tilted initial marginal of the unrestricted bridge, this constrained optimum coincides with the unrestricted one. Otherwise, the actor does not reweight the initial latent distribution in a state-dependent way, and the corollary quantifies the resulting gap.

\begin{corollary}[Base-constraint gap controls endpoint deviation]
Under the assumptions of Theorem~\ref{thm:fixed_base_optimum}, the loss relative to the unrestricted optimum is
\begin{equation}
  \Delta(s)
  =
  \alpha
  \left(
    \log \bar Z_Q(s)
    -
    \mathbb E_{z_0\sim p_0}
    \left[
      \log Z_Q(s,z_0)
    \right]
    +
    D_{\mathrm{KL}}
    \left(
      p_0
      \,\|\,
      r_0
    \right)
  \right)
  \ge 0.
  \label{eq:base_jensen_gap}
\end{equation}
Equality holds if and only if $p_0$ equals the initial marginal of the unrestricted tilted bridge. In the special case $p_0=r_0$, this reduces to $Z_Q(s,z_0)$ being constant for $p_0$-almost every $z_0$. Moreover,
\begin{equation}
  D_{\mathrm{KL}}
  \left(
    P_{Q,p_0}^\star(\cdot\mid s)
    \,\|\,P_Q^\star(\cdot\mid s)
  \right)
  =
  \frac{
    \Delta(s)
  }{
    \alpha
  },
  \label{eq:shared_base_path_bound}
\end{equation}
and the induced endpoint deviation obeys
\begin{equation}
  D_{\mathrm{KL}}
  \left(
    \pi_{Q,p_0}^\star(\cdot\mid s)
    \,\|\,\pi_Q^\star(\cdot\mid s)
  \right)
  \le
  \frac{
    \Delta(s)
  }{
    \alpha
  }.
  \label{eq:shared_base_endpoint_bound}
\end{equation}
\end{corollary}

The proof is given in Appendix~\ref{app:proof_base_gap}. The unrestricted optimum can tilt the initial latent law by $Z_Q(s,z_0)$, which would require a state-dependent base sampler before any bridge transition. A fixed-base actor avoids this sampler. Its cost is $\Delta(s)$, and we show that this same cost controls the endpoint deviation. The bound is small unless $Z_Q(s,z_0)$ varies sharply across likely base latents, meaning some base regions are much better positioned to reach high-value actions than others. In the regime we target, typical base samples can be transported to useful endpoints. The base choice then acts mainly as a practical inductive bias, while the bridge transitions perform value-guided transport. For a finite-step implementation, the reference terminal law is $u_R$ rather than the ideal uniform action law. The endpoint-MaxEnt connection is exact with respect to $u_R$ and recovers the uniform-prior view when $u_R=u$. What \myalgo optimizes is more concrete. For the chosen actor-reference pair, the finite-step path KL is the exact regularizer up to the fixed initial-base constant. We now need an architecture that exposes this KL from sampled paths.

\section{Path-Regularized Generative Actor-Critic}

The preceding section turns soft policy improvement into a path-law optimization problem with a fixed base distribution. This leaves an architectural question. How can the actor expose the path likelihood needed by the regularizer while remaining cheap to sample? We answer this question with \myalgo, a path-regularized actor-critic that combines a short bridge actor with an off-policy soft critic. The bridge actor makes the finite-step path KL factor into local Gaussian terms. These terms will become the sampled control energy used by the actor and critic updates.

\subsection{Soft Bridge Policies}

A soft bridge policy is a finite Markov chain in pre-tanh latent space. It starts from a fixed base law and applies $K$ lightweight Gaussian residual transitions before mapping the terminal latent state to the bounded action. This gives an explicit path law and avoids the long shared-parameter sampler used by high-NFE iterative policies. We write its transition law as
\begin{equation}
  z_0\sim p_0,
  \qquad
  q_{\theta,k}(z_{k+1}\mid z_k,s)
  =
  \mathcal N
  \left(
    z_k+h\,u_{\theta,k}(s,z_k),
    2h\,\mathrm{diag}(\sigma_{\theta,k}^2(s,z_k))
  \right),
  \label{eq:bridge_actor_kernel}
\end{equation}
for $k=0,\ldots,K-1$ and $h=1/K$. The terminal action is $a=T(z_K)=\tanh(z_K)$ after the usual affine rescaling to the environment action bounds. Each transition block has its own parameters and communicates with the next block only through the action-dimensional latent state $z_k$. Thus the forward pass is a stochastic generative process, but it is still a fixed-depth one-pass network rather than an iterative sampler that reuses the same refinement model many times. During actor updates, gradients pass through this short sequence of transition blocks instead of a long BPTT graph over repeated uses of a shared sampler.

The reference bridge plays the role of a high-entropy action prior. In normalized action space $(-1,1)^{d_a}$, the pre-tanh density whose tanh image is uniform is
\begin{equation}
  q_{\mathrm{ref}}(z)
  =
  \frac{1}{2^{d_a}}
  \prod_{i=1}^{d_a}
  \mathrm{sech}^2(z_i),
  \qquad
  \nabla_z\log q_{\mathrm{ref}}(z)
  =
  -2\tanh(z).
\end{equation}
For the main implementation, we set the actor base law to this reference density, $p_0=q_{\mathrm{ref}}$. This keeps the full path KL free of an initial constant and gives the cleanest decomposition below. Other fixed bases are also valid. They add a parameter-independent initial KL to the full objective and mainly act as an implementation-level inductive bias through the base samples. The finite-step reference starts from the logistic reference base and uses an Euler Gaussian kernel based on the score above,
\begin{equation}
  r_k(z_{k+1}\mid z_k)
  =
  \mathcal N
  \left(
    z_k-2h\tanh(z_k),
    2hI
  \right).
  \label{eq:bridge_reference_kernel}
\end{equation}
The continuous reference diffusion is stationary when initialized from $q_{\mathrm{ref}}$, so its tanh image is uniform in action space at every time. The finite-step reference is the object used by the algorithm. The path KL below is exact for this bridge pair, while Appendix~\ref{app:finite_reference_bias} quantifies its endpoint bias.

\begin{lemma}[Finite-step control-energy decomposition]
\label{lem:control_energy}
Let
\begin{equation}
  P_\theta(d\tau\mid s)
  =
  p_0(dz_0)
  \prod_{k=0}^{K-1}
  q_{\theta,k}(dz_{k+1}\mid z_k,s),
  \qquad
  R(d\tau)
  =
  r_0(dz_0)
  \prod_{k=0}^{K-1}
  r_k(dz_{k+1}\mid z_k).
\end{equation}
Then
\begin{equation}
  D_{\mathrm{KL}}
  \left(
    P_\theta(\cdot\mid s)
    \,\|\,R
  \right)
  =
  D_{\mathrm{KL}}
  \left(
    p_0
    \,\|\,
    r_0
  \right)
  +
  \sum_{k=0}^{K-1}
  \mathbb E_{\tau\sim P_\theta}
  \left[
    D_{\mathrm{KL}}
    \left(
      q_{\theta,k}(\cdot\mid z_k,s)
      \,\|\,r_k(\cdot\mid z_k)
    \right)
  \right].
  \label{eq:local_kl_sum}
\end{equation}
For Gaussian local kernels, define $\mathcal C_\theta(s,\tau)$ as the sum of the local Gaussian KL terms. Then
\begin{equation}
  D_{\mathrm{KL}}
  \left(
    P_\theta(\cdot\mid s)
    \,\|\,R
  \right)
  =
  D_{\mathrm{KL}}
  \left(
    p_0
    \,\|\,
    r_0
  \right)
  +
  \mathbb E_{\tau\sim P_\theta(\cdot\mid s)}
  \left[
    \mathcal C_\theta(s,\tau)
  \right].
  \label{eq:path_kl_control_energy}
\end{equation}
\end{lemma}

The proof is a factorization of the Markov path likelihood ratio and appears in Appendix~\ref{app:proof_control_energy}. When $p_0=r_0$, the initial KL vanishes. When $p_0$ is a different fixed base, this initial term is constant with respect to the actor transitions, so actor training can use the same sampled transition control energy. For the residual actor in Eq.~\eqref{eq:bridge_actor_kernel} and the reference in Eq.~\eqref{eq:bridge_reference_kernel}, the local control cost is
\begin{equation}
  \mathcal C_k(s,z_k)
  =
  \frac{1}{2}
  \sum_{i=1}^{d_a}
  \left[
    \sigma_{\theta,k,i}^2(s,z_k)
    +
    \frac{
      \left(
        \mu_{\theta,k,i}(s,z_k)
        -
        \mu_{R,k,i}(z_k)
      \right)^2
    }{
      2h
    }
    -
    1
    -
    \log \sigma_{\theta,k,i}^2(s,z_k)
  \right],
  \label{eq:local_control_cost}
\end{equation}
where $\mu_{\theta,k}(s,z_k)=z_k+h\,u_{\theta,k}(s,z_k)$ and $\mu_{R,k}(z_k)=z_k-2h\tanh(z_k)$. The sampled path cost is $\mathcal C_\theta(s,\tau)=\sum_{k=0}^{K-1}\mathcal C_k(s,z_k)$. This quantity is the finite-step transition control energy of the bridge path, and its expectation is the actor-reference path KL up to the fixed initial-base constant. The soft regularizer in \myalgo is therefore an analytical path-wise relative entropy rather than an estimate, bound or proxy for endpoint action entropy. The next subsection uses this finite-step path regularizer inside an off-policy actor-critic update.

\subsection{Off-Policy Actor-Critic Training}

We train this bridge actor with an off-policy critic in \myalgo. For a fixed bridge policy, define the path-regularized soft value
\begin{equation}
  V^\pi(s)
  =
  \mathbb E_{\tau\sim P_\theta(\cdot\mid s)}
  \left[
    Q^\pi(s,T(z_K))
    -
    \alpha\,\mathcal C_\theta(s,\tau)
  \right].
  \label{eq:soft_bridge_value}
\end{equation}
The corresponding Bellman equation is
\begin{equation}
  Q^\pi(s,a)
  =
  r(s,a)
  +
  \gamma
  \mathbb E_{s'}
  \left[
    V^\pi(s')
  \right].
  \label{eq:soft_bridge_bellman}
\end{equation}
Thus the only change from a standard soft actor-critic update is that endpoint entropy is replaced by sampled bridge control energy. Given a replay batch $(s,a,r,s',d)$, we sample a next-state bridge path $\tau'\sim P_{\bar\theta}(\cdot\mid s')$ from the target actor and form the soft bootstrap scalar
\begin{equation}
  y
  =
  r
  +
  \gamma(1-d)
  \left[
    Q_{\bar\phi}^{\min}(s',T(z'_K))
    -
    \alpha\,\mathcal C_{\bar\theta}(s',\tau')
  \right].
  \label{eq:soft_bridge_target}
\end{equation}
The critic can be any off-policy estimator trained toward this target. Our implementation uses a twin categorical critic with a fixed support $\mathcal Z=\{v_1,\ldots,v_M\}$ and a CrossQ-style update~\cite{bellemare2017distributional, bhatt2024crossq,DIME25}. For each critic head, the target distribution is obtained by shifting the support by the sampled soft bootstrap and projecting back to $\mathcal Z$ with the usual categorical projection. This critic choice is an implementation detail rather than a requirement of the bridge objective. The actor is updated by sampling current-state paths $\tau\sim P_\theta(\cdot\mid s)$ and minimizing
\begin{equation}
  \mathcal L_{\mathrm{actor}}(\theta)
  =
  \mathbb E_{s\sim\mathcal D}
  \mathbb E_{\tau\sim P_\theta(\cdot\mid s)}
  \left[
    \alpha\,\mathcal C_\theta(s,\tau)
    -
    Q_\phi^{\min}(s,T(z_K))
  \right].
  \label{eq:softgac_actor_loss}
\end{equation}
At a fixed state, this sampled objective has a precise projection interpretation. It moves the transition actor toward the ideal soft-optimal bridge, while the chosen fixed base determines the practical actor family being searched.
\begin{proposition}[Actor update as restricted projection to the ideal bridge]
\label{prop:actor_reverse_kl}
Let $P_Q^\star$ be the unrestricted soft-optimal bridge in Theorem~1. For any finite-step bridge actor with fixed base law $p_0\ll r_0$ that satisfies Lemma~\ref{lem:control_energy},
\begin{equation}
  \mathcal L_{\mathrm{actor}}(\theta\mid s)
  =
  \mathbb E_{\tau\sim P_\theta(\cdot\mid s)}
  \left[
    \alpha\,\mathcal C_\theta(s,\tau)
    -
    Q(s,T(z_K))
  \right]
\end{equation}
satisfies
\begin{equation}
  \mathcal L_{\mathrm{actor}}(\theta\mid s)
  =
  \alpha
  D_{\mathrm{KL}}
  \left(
    P_\theta(\cdot\mid s)
    \,\|\,P_Q^\star(\cdot\mid s)
  \right)
  -
  \alpha\log Z_Q(s)
  -
  \alpha
  D_{\mathrm{KL}}
  \left(
    p_0
    \,\|\,
    r_0
  \right).
  \label{eq:actor_reverse_kl}
\end{equation}
Consequently, minimizing this loss over the unrestricted fixed-base path-law family has global minimizer $P_{Q,p_0}^\star$ from Theorem~\ref{thm:fixed_base_optimum}.
\end{proposition}
The proof appears in Appendix~\ref{app:proof_actor_projection}. The global-minimizer statement is a property of the unrestricted fixed-base path-law family. The finite neural Gaussian Markov actor optimized by SGD is a restricted parameterization, so the result should not be read as a global optimization guarantee for the implemented network. Equation~\eqref{eq:actor_reverse_kl} is nevertheless useful algorithmically. Once the base law is fixed, the final two terms are constants with respect to the transition actor parameters, and the sampled actor loss follows a reverse-KL projection direction toward the ideal bridge. In the implemented neural family, the loss defines a tractable finite-step objective whose attainable solution is limited by parameterization and optimization. 

In addition, we tune the temperature with a SAC-style dual update on the control-energy budget. Let $\mathcal C_{\mathrm{target}} = \rho_{\mathrm{ctrl}} K d_a$ be a per-step, per-dimension heuristic target cost. Appendix~\ref{app:rho_ctrl_choice} gives the scale intuition behind this choice. The temperature objective is
\begin{equation}
  \mathcal L_\alpha
  =
  \mathbb E
  \left[
    \alpha
    \left(
      \mathcal C_{\mathrm{target}}
      -
      \mathcal C_\theta(s,\tau)
    \right)
  \right],
  \qquad
  \alpha=\exp(\log\alpha).
  \label{eq:alpha_loss}
\end{equation}
This keeps the average bridge deviation from the high-entropy reference near a chosen budget. Action generation remains fixed-cost with exactly $K$ local transition blocks. The complete training loop is given in Algorithm~\ref{alg:softgac} in Appendix~\ref{app:pseudocode}.

\section{Experiments}
\paragraph{Experimental setup.} 
We evaluate \myalgo on challenging high-dimensional continuous-control tasks from the DeepMind Control Suite (DMC)~\cite{tassa2018deepmind,dm_control2020} and HumanoidBench~\cite{sferrazza2024humanoidbench}. We measure both return and the compute-return tradeoff. Generative baselines include FLAC~\cite{FLAC26}, DIME~\cite{DIME25}, FlowRL~\cite{FlowRL25}, QSM~\cite{QSM24}, QVPO~\cite{QVPO24}, DACER~\cite{ACER24}, MaxEntDP~\cite{MaxEntDP25}, SDAC~\cite{SDAC25} and a conditional normalizing-flow actor~\cite{NormalizingFlowsCC25}. CrossQ-SAC~\cite{bhatt2024crossq} provides the diagonal-Gaussian baseline. We report interquartile mean (IQM)~\cite{agarwal2021deep} over $8$ random seeds with $95\%$ bootstrap confidence intervals. Ablations isolate the contribution of the path-space regularizer.

We re-implement all baselines in a single codebase under the framework of StableBaselines3~\cite{stable-baselines3}, and we follow the official implementations whenever available. All methods share the same infrastructure and main critic update. We adjust actor sizes to obtain comparable parameter budgets, while the reported latency measures the complete action-generation computation for each actor. Our \myalgo actor uses $K=6$ bridge transitions, each represented by a lightweight single-layer MLP with mean and variance heads. For deterministic evaluation and deployment, our \myalgo actor follows the mean bridge path, bypassing the variance heads and omitting the control-energy computation. Appendices~\ref{app:experimental_setup} and~\ref{app:more_results} give the full implementation details, hyperparameters, experimental setup, extended results and ablations.

\begin{figure}[!htbp]
  \centering
  \includegraphics[width=0.9\linewidth]{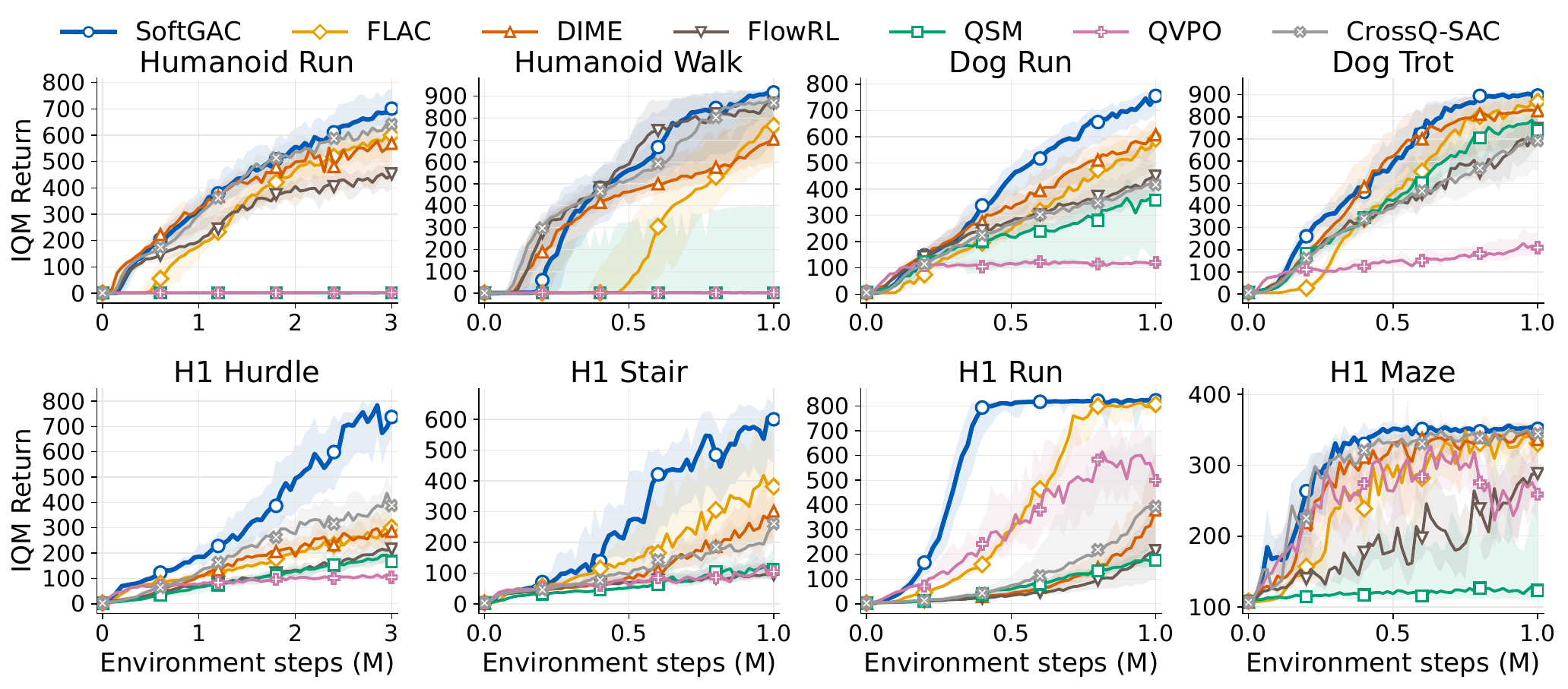}
  \caption{IQM learning curves on the 8 hard control tasks.}
  \label{fig:iqm_learning_curves_hard8}
\end{figure}

\paragraph{Performance.}
Figure~\ref{fig:iqm_learning_curves_hard8} shows that \myalgo gives the strongest overall performance on the hard tasks. The gains are largest on high-dimensional, long-horizon locomotion tasks such as Humanoid Run, Dog Run, H1 Hurdle and H1 Stair, where \myalgo improves over both generative baselines and the CrossQ-SAC Gaussian baseline. On tasks where CrossQ-SAC is already strong, such as Humanoid Walk and H1 Maze, \myalgo remains competitive or better while retaining a richer generative actor. These results suggest that the soft bridge actor improves sample efficiency and final performance, rather than only increasing expressiveness at convergence.

\begin{figure}[!htbp]
  \centering
  \includegraphics[width=0.89\linewidth]{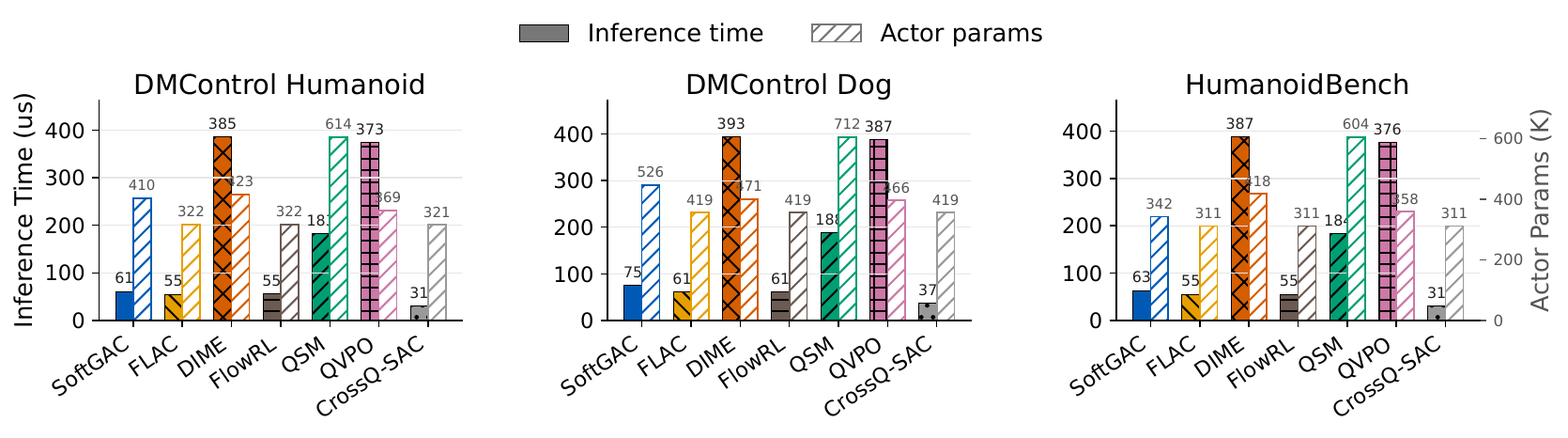}
  \caption{Per-action inference time and actor parameter count measured on an Apple M3 Pro CPU.}
  \label{fig:inference_time_by_domain}
\end{figure}

\paragraph{Inference cost.}
Figure~\ref{fig:inference_time_by_domain} shows that this return gain does not come from unrolled sampler evaluations. CrossQ-SAC has the cheapest unimodal Gaussian actor. \myalgo uses about $61$--$75\,\mu$s per action on the CPU benchmark, which places it in the same low-latency range as one-step flow baselines and far below high-NFE diffusion baselines. This matters because actor sampling is used both during environment interaction and inside actor-critic updates. The soft bridge policy generates each action with one sampled pass through its bridge blocks and keeps the actor parameter budget comparable, while still offering superior performance. This explains the improved compute-return tradeoff observed in Figures~\ref{fig:iqm_learning_curves_hard8} and~\ref{fig:inference_time_by_domain}.

\begin{figure}[!htbp]
  \centering
  \includegraphics[width=0.95\linewidth]{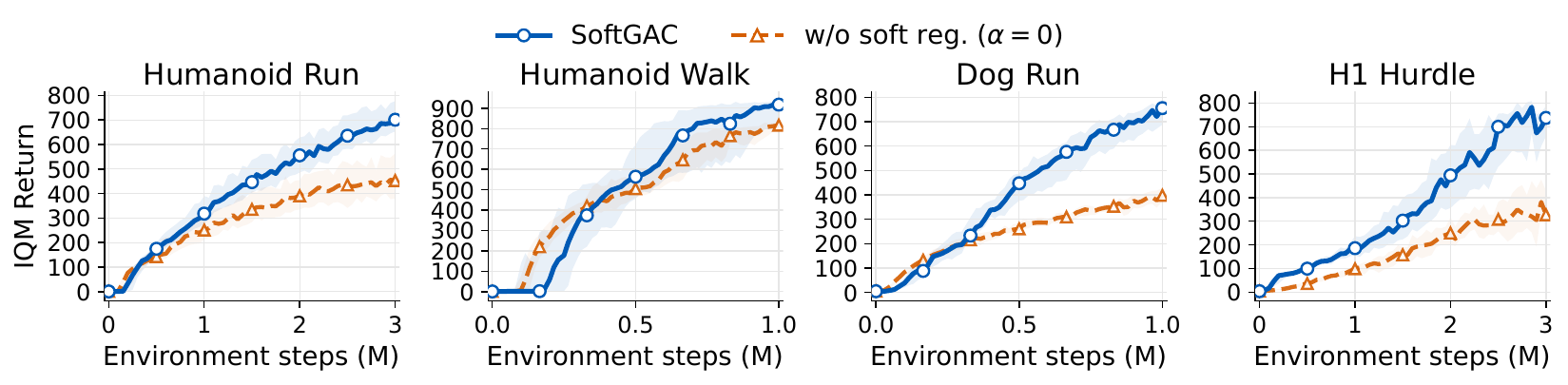}
  \caption{Ablation study of the soft regularizer.}
  \label{fig:soft_regularizer_ablation_curves}
\end{figure}

\paragraph{Soft regularization.}
Figure~\ref{fig:soft_regularizer_ablation_curves} isolates the path-space soft regularizer by setting $\alpha=0$ while keeping the same actor and critic. Removing this term consistently hurts Humanoid Run, Humanoid Walk, Dog Run and H1 Hurdle, which suggests that the gains do not come only from the bridge parameterization. The result supports the relative-entropy control energy as a principled soft objective for the one-pass generator, rather than only an exploration heuristic.

\section{Related Work and Discussion}


\paragraph{Expressive generative policies.}
Generative policies broaden the actor classes available to continuous-control RL beyond the diagonal Gaussians commonly used in SAC~\citep{SAC18}. Normalizing flows offer richer distributions while retaining tractable densities. Energy-based normalizing flows integrate policy and value learning into a unified formulation~\citep{NormalizingFlowRL24}, while subsequent work demonstrates the competitiveness of normalizing-flow policies across continuous-control settings~\citep{NormalizingFlowsCC25}. However, exact action-density evaluation in normalizing-flow policies relies on invertible transformations with tractable Jacobian determinants, limiting the architectural flexibility of the actor. Stein soft actor-critic offers a different construction based on iterative particle transformations and an explicit entropy approximation~\citep{messaoud2024s}. Diffusion and flow-matching policies support several reward-driven learning schemes. Diffusion policy gradient learns multimodal behaviors from scratch~\citep{DDiffPG24}, Q-score matching uses action gradients of the critic to guide score learning~\citep{QSM24} and QVPO constructs a Q-weighted variational policy optimization objective~\citep{QVPO24}. Flow matching policy gradients develop on-policy updates from flow-matching losses~\citep{FPO25}. For pretrained policies, DPPO and ReinFlow study reinforcement-learning fine-tuning of diffusion and flow models~\citep{DPPO25,ReinFlow25}. For MaxEnt RL, however, richer actors also raise the two challenges of tractable soft regularization and efficient action generation and optimization.

\paragraph{Soft regularization for implicit policies.}
The obstacle to soft regularization is that standard soft policy improvement and evaluation require the marginal action log density, which implicit generators generally do not expose. DIME derives a tractable entropy lower bound from diffusion paths and incorporates it into soft policy iteration~\citep{DIME25}, while SAC Flow uses tractable joint likelihoods along noise-augmented generation paths~\citep{SACFlow26}. Endpoint entropy and log-density estimates provide another route. DACER fits a Gaussian mixture model to sampled actions and uses its entropy estimate to regulate exploration~\citep{ACER24}. MaxEntDP estimates the Boltzmann score through self-normalized importance sampling and uses Monte Carlo endpoint log-density estimation in its soft Bellman backup~\citep{MaxEntDP25}, whereas SDAC combines reverse-sampling reweighted score matching for Boltzmann policy improvement with an additive-noise log-density approximation for soft evaluation~\citep{SDAC25}. Transport and bridge formulations shift the focus from density evaluation to how distributions are connected. FlowRL constrains policy improvement through Wasserstein distance to a replay-derived reference policy~\citep{FlowRL25}. FLAC is particularly close to our work in motivation, formulating MaxEnt RL through a generalized Schr\"odinger bridge and kinetic-energy regularization~\citep{FLAC26}. In its deterministic flow regime, however, kinetic energy is a geometric proxy that does not by itself imply high terminal entropy. WPPG explicitly supplies the entropy contribution through a heat-flow step after Wasserstein proximal transport, while its $W_2$ term remains an action-space proximal metric~\citep{WPPG26}. In \myalgo, we directly optimize the actor-reference path KL for the implemented finite-step bridge. Unlike a purely geometric proximity penalty, this KL explicitly contains endpoint KL to the reference action law. For the Gaussian transitions, its trainable part reduces exactly to sampled transition control energy, including both mean and variance terms. We use this same analytical local regularizer in actor improvement and the soft Bellman bootstrap, without fitting the endpoint density or introducing a separate entropy estimator.

\paragraph{Training and inference costs.}
Tractable regularization alone does not remove the cost of generating and optimizing actions. Many diffusion and flow policies repeatedly evaluate a shared network, so higher NFE directly raises inference cost~\citep{DIME25,SACFlow26}. When actor updates differentiate $Q(s,a)$ through sampled actions, higher NFE also increases the depth of unrolled backpropagation, activation memory and actor-update cost~\citep{FQL25,OFQL26}. Long recursive gradient paths can further destabilize optimization, motivating the velocity-reparameterized sequential models in SAC Flow~\citep{SACFlow26}. Simply reducing the number of sampling steps, however, can sacrifice policy quality~\citep{OFQL26}. FQL instead trains a separate one-step actor with distillation from a flow behavior model, avoiding recursive Q-gradient backpropagation while still requiring teacher sampling during training~\citep{FQL25}. One-Step FQL removes the separate distillation stage by learning an average-velocity formulation for direct one-step offline policy generation~\citep{OFQL26}. Complementary approaches include shortcut models for one-step or few-step generation~\citep{fransone} and D3P for state-dependent denoising depth~\citep{D3P25}. \myalgo shares the goal of inexpensive action generation but derives its one-pass actor from the online soft objective itself. We use a compact sequence of step-specific bridge blocks rather than repeated evaluations of a shared sampler, so each sampled forward pass both generates an action and exposes the local transition-control terms. Under a comparable total actor parameter budget, this design addresses soft regularization and actor computation together, without distillation or long recursive backpropagation through a shared sampler.

\section{Conclusion and Limitations}
We presented \myalgo, a generative actor-critic method that lifts maximum-entropy regularization from endpoint actions to the full latent generation path. This path-space view yields an analytical relative-entropy regularizer, implemented as sampled transition control energy for Gaussian bridges. It also gives a single-pass actor that avoids repeatedly applying a high-NFE shared sampler and long BPTT through that sampler. Across challenging locomotion benchmarks, \myalgo achieves higher or competitive returns while remaining in the low-latency regime of one-pass actors, which supports soft bridge policies as a practical design for efficient generative actor-critic learning.

Compared with Gaussian SAC, our bridge policy introduces additional design choices for depth, transition variance and control budget, although our ablations indicate robustness across the tested settings. Our empirical evaluation also focuses on state-based continuous control in simulation. These considerations motivate future work on bridge architectures that scale effectively with depth and width, as well as image-conditioned policies, integration with world models and real-robot evaluation.

\begin{ack}
This work was supported by the Guangdong Provincial Key Research and Development Program under Grant 2024B0101040006 and by the National Natural Science Foundation of China (NSFC) under Grants U23A20273, 62271158 and 625B2165. This work was also supported in part by the Google TPU Research Cloud (TRC) program, which provided access to TPU compute resources for the experiments. The authors declare no competing interests.
\end{ack}

\bibliographystyle{unsrtnat}
\bibliography{ref}

\clearpage

\tableofcontents
\clearpage

\appendix

\section{Theorem Proofs}
\label{app:proofs}

\subsection{Proof of Proposition~1}
\label{app:proof_path_lift}

\paragraph{Restatement.}
Let $a=T(z_K)$ be the terminal action. If
\[
  P_\theta(d\tau\mid s)
  =
  \pi_\theta(da\mid s)\,P_\theta(d\tau\mid s,a),
  \qquad
  R(d\tau)
  =
  u_R(da)\,R(d\tau\mid a),
\]
then
\[
  D_{\mathrm{KL}}
  \left(
    P_\theta(\cdot\mid s)
    \,\|\,R
  \right)
  =
  D_{\mathrm{KL}}
  \left(
    \pi_\theta(\cdot\mid s)
    \,\|\,u_R
  \right)
  +
  \mathbb E_{a\sim\pi_\theta}
  \left[
    D_{\mathrm{KL}}
    \left(
      P_\theta(\cdot\mid s,a)
      \,\|\,R(\cdot\mid a)
    \right)
  \right].
\]

\begin{proof}
By the stated disintegrations and the terminal map $a=T(z_K)$, the Radon-Nikodym derivative factorizes into an endpoint term and a conditional path term:
\begin{equation}
  \log
  \frac{
    dP_\theta(\tau\mid s)
  }{
    dR(\tau)
  }
  =
  \log
  \frac{
    d\pi_\theta(a\mid s)
  }{
    du_R(a)
  }
  +
  \log
  \frac{
    dP_\theta(\tau\mid s,a)
  }{
    dR(\tau\mid a)
  }.
\end{equation}
This is the chain rule for relative entropy written at the level of path laws. Taking expectation under $P_\theta(\cdot\mid s)$ and then disintegrating the expectation by $a\sim\pi_\theta(\cdot\mid s)$ gives
\begin{align}
  D_{\mathrm{KL}}
  \left(
    P_\theta(\cdot\mid s)
    \,\|\,R
  \right)
  &=
  \mathbb E_{a\sim\pi_\theta}
  \left[
    \log
    \frac{
      d\pi_\theta(a\mid s)
    }{
      du_R(a)
    }
  \right]
  \nonumber\\
  &\quad+
  \mathbb E_{a\sim\pi_\theta}
  \mathbb E_{\tau\sim P_\theta(\cdot\mid s,a)}
  \left[
    \log
    \frac{
      dP_\theta(\tau\mid s,a)
    }{
      dR(\tau\mid a)
    }
  \right].
\end{align}
The first term is $D_{\mathrm{KL}}(\pi_\theta(\cdot\mid s)\,\|\,u_R)$. For each terminal action $a$, the inner expectation in the second term is $D_{\mathrm{KL}}(P_\theta(\cdot\mid s,a)\,\|\,R(\cdot\mid a))$. This proves Eq.~\eqref{eq:path_kl_decomposition}.
\end{proof}

\subsection{Proof of Theorem~1}
\label{app:proof_unrestricted_bridge}

\paragraph{Restatement.}
For fixed $s$, $Q$ and $\alpha>0$, optimize Eq.~\eqref{eq:path_soft_objective} over all path laws $P\ll R$. If
\[
  Z_Q(s)
  =
  \mathbb E_{\tau\sim R}
  \left[
    \exp
    \left(
      Q(s,T(z_K))/\alpha
    \right)
  \right]
  <
  \infty,
\]
then the unique optimum is
\[
  \frac{dP_Q^\star}{dR}(\tau\mid s)
  =
  \frac{
    \exp
    \left(
      Q(s,T(z_K))/\alpha
    \right)
  }{
    Z_Q(s)
  }.
\]
Its terminal action law is
\[
  \pi_Q^\star(a\mid s)
  =
  \frac{
    u_R(a)
    \exp
    \left(
      Q(s,a)/\alpha
    \right)
  }{
    \int_{\mathcal A}
    u_R(b)
    \exp
    \left(
      Q(s,b)/\alpha
    \right)
    db
  }.
\]

\begin{proof}
Let $f_s(\tau)=Q(s,T(z_K))$. Since $Z_Q(s)<\infty$, the density in Eq.~\eqref{eq:unrestricted_tilted_bridge} integrates to one under $R$, so it defines a valid path law $P_Q^\star$ with $P_Q^\star\ll R$. For any admissible path law $P\ll R$,
\begin{equation}
  D_{\mathrm{KL}}
  \left(
    P
    \,\|\,P_Q^\star
  \right)
  =
  \mathbb E_{\tau\sim P}
  \left[
    \log\frac{dP}{dR}(\tau)
    -
    \frac{f_s(\tau)}{\alpha}
    +
    \log Z_Q(s)
  \right].
\end{equation}
Rearranging yields the Gibbs variational identity
\begin{equation}
  \mathbb E_{\tau\sim P}[f_s(\tau)]
  -
  \alpha
  D_{\mathrm{KL}}
  \left(
    P
    \,\|\,R
  \right)
  =
  \alpha\log Z_Q(s)
  -
  \alpha
  D_{\mathrm{KL}}
  \left(
    P
    \,\|\,P_Q^\star
  \right).
\end{equation}
The left-hand side is exactly the unrestricted path objective evaluated at $P$. The right-hand side is upper bounded by $\alpha\log Z_Q(s)$ because relative entropy is nonnegative. The upper bound is attained if and only if $D_{\mathrm{KL}}(P\,\|\,P_Q^\star)=0$, which holds if and only if $P=P_Q^\star$ almost surely. This proves both optimality and uniqueness.

To compute the terminal marginal, disintegrate the reference as $R(d\tau)=u_R(da)R(d\tau\mid a)$. Since $f_s(\tau)=Q(s,a)$ depends only on the terminal action,
\begin{equation}
  P_Q^\star(da\mid s)
  =
  \frac{
    \exp(Q(s,a)/\alpha)
  }{
    Z_Q(s)
  }
  u_R(da).
\end{equation}
The normalizer can be written as
\begin{equation}
  Z_Q(s)
  =
  \int_{\mathcal A}
  u_R(b)
  \exp(Q(s,b)/\alpha)
  db,
\end{equation}
because $\int R(d\tau\mid b)=1$ for every terminal action $b$. Substituting this expression gives Eq.~\eqref{eq:unrestricted_endpoint}.
\end{proof}

\subsection{Proof of Theorem~2}
\label{app:proof_fixed_base_optimum}

\paragraph{Restatement.}
Let
\[
  Z_Q(s,z_0)
  =
  \mathbb E_{\tau\sim R(\cdot\mid z_0)}
  \left[
    \exp
    \left(
      Q(s,T(z_K))/\alpha
    \right)
  \right],
  \qquad
  \bar Z_Q(s)
  =
  \mathbb E_{z_0\sim r_0}
  \left[
    Z_Q(s,z_0)
  \right].
\]
Among all path laws with initial marginal $p_0$, the unique optimum is
\[
  P_{Q,p_0}^\star(d\tau\mid s)
  =
  p_0(dz_0)\,
  \frac{
    \exp
    \left(
      Q(s,T(z_K))/\alpha
    \right)
  }{
    Z_Q(s,z_0)
  }
  R(d\tau_{1:K}\mid z_0).
\]
Its objective value is
\[
  \alpha
  \mathbb E_{z_0\sim p_0}
  \left[
    \log Z_Q(s,z_0)
  \right]
  -
  \alpha
  D_{\mathrm{KL}}
  \left(
    p_0
    \,\|\,
    r_0
  \right).
\]

\begin{proof}
Any feasible fixed-base law can be written as
\begin{equation}
  P(d\tau\mid s)
  =
  p_0(dz_0)
  P(d\tau_{1:K}\mid z_0,s).
\end{equation}
The reference has the corresponding decomposition
\begin{equation}
  R(d\tau)
  =
  r_0(dz_0)
  R(d\tau_{1:K}\mid z_0).
\end{equation}
The likelihood ratio therefore separates into an initial-base term and a conditional path term:
\begin{equation}
  \log
  \frac{
    dP
  }{
    dR
  }(\tau\mid s)
  =
  \log
  \frac{
    dp_0
  }{
    dr_0
  }(z_0)
  +
  \log
  \frac{
    dP(\cdot\mid z_0,s)
  }{
    dR(\cdot\mid z_0)
  }(\tau_{1:K}).
\end{equation}
Taking expectation under $P$ gives
\begin{equation}
  D_{\mathrm{KL}}
  \left(
    P
    \,\|\,R
  \right)
  =
  D_{\mathrm{KL}}
  \left(
    p_0
    \,\|\,
    r_0
  \right)
  +
  \mathbb E_{z_0\sim p_0}
  \left[
    D_{\mathrm{KL}}
    \left(
      P(\cdot\mid z_0,s)
      \,\|\,R(\cdot\mid z_0)
    \right)
  \right].
\end{equation}
The value term also decomposes conditionally:
\begin{equation}
  \mathbb E_{\tau\sim P}
  \left[
    Q(s,T(z_K))
  \right]
  =
  \mathbb E_{z_0\sim p_0}
  \mathbb E_{\tau_{1:K}\sim P(\cdot\mid z_0,s)}
  \left[
    Q(s,T(z_K))
  \right].
\end{equation}
Hence the fixed-base objective is an average, over $z_0\sim p_0$, of independent conditional variational problems minus the initial KL constant. For each fixed $z_0$, define
\begin{equation}
  f_{s,z_0}(\tau_{1:K})
  =
  Q(s,T(z_K)).
\end{equation}
Applying the Gibbs variational identity to the conditional reference $R(\cdot\mid z_0)$ gives
\begin{align}
  \Psi_{z_0}(P)
  &:=
  \mathbb E_{\tau_{1:K}\sim P(\cdot\mid z_0,s)}
  \left[
    f_{s,z_0}(\tau_{1:K})
  \right]
  -
  \alpha
  D_{\mathrm{KL}}
  \left(
    P(\cdot\mid z_0,s)
    \,\|\,R(\cdot\mid z_0)
  \right)
  \nonumber\\
  &=
  \alpha\log Z_Q(s,z_0)
  -
  \alpha
  D_{\mathrm{KL}}
  \left(
    P(\cdot\mid z_0,s)
    \,\|\,P^\star(\cdot\mid z_0,s)
  \right),
\end{align}
where the unique conditional optimizer satisfies
\begin{equation}
  \frac{
    dP^\star(\cdot\mid z_0,s)
  }{
    dR(\cdot\mid z_0)
  }(\tau_{1:K})
  =
  \frac{
    \exp(Q(s,T(z_K))/\alpha)
  }{
    Z_Q(s,z_0)
  },
\end{equation}
and where $Z_Q(s,z_0)$ is the conditional partition function in Theorem~\ref{thm:fixed_base_optimum}. The conditional KL term is nonnegative and vanishes only at this optimizer, so the optimizer is unique for $p_0$-almost every $z_0$. Averaging the conditional value $\alpha\log Z_Q(s,z_0)$ under $p_0$ and subtracting $\alpha D_{\mathrm{KL}}(p_0\,\|\,r_0)$ gives Eq.~\eqref{eq:fixed_base_objective}. Substituting the conditional optimizer into the fixed-base disintegration gives Eq.~\eqref{eq:shared_base_optimum}.
\end{proof}

\subsection{Proof of Corollary~1}
\label{app:proof_base_gap}

\paragraph{Restatement.}
Under the assumptions of Theorem~\ref{thm:fixed_base_optimum}, the loss relative to the unrestricted optimum is
\[
  \Delta(s)
  =
  \alpha
  \left(
    \log \bar Z_Q(s)
    -
    \mathbb E_{z_0\sim p_0}
    \left[
      \log Z_Q(s,z_0)
    \right]
    +
    D_{\mathrm{KL}}
    \left(
      p_0
      \,\|\,
      r_0
    \right)
  \right)
  \ge 0.
\]
Equality holds if and only if $p_0$ equals the initial marginal of the unrestricted tilted bridge. If $p_0=r_0$, this is equivalent to $Z_Q(s,z_0)$ being constant for $p_0$-almost every $z_0$. Moreover,
\[
  D_{\mathrm{KL}}
  \left(
    P_{Q,p_0}^\star(\cdot\mid s)
    \,\|\,P_Q^\star(\cdot\mid s)
  \right)
  =
  \frac{
    \Delta(s)
  }{
    \alpha
  },
\]
and
\[
  D_{\mathrm{KL}}
  \left(
    \pi_{Q,p_0}^\star(\cdot\mid s)
    \,\|\,\pi_Q^\star(\cdot\mid s)
  \right)
  \le
  \frac{
    \Delta(s)
  }{
    \alpha
  }.
\]

\begin{proof}
The unrestricted partition function can be decomposed by the reference base law:
\begin{align}
  Z_Q(s)
  &=
  \mathbb E_{\tau\sim R}
  \left[
    \exp(Q(s,T(z_K))/\alpha)
  \right]
  \nonumber\\
  &=
  \mathbb E_{z_0\sim r_0}
  \left[
    Z_Q(s,z_0)
  \right]
  \nonumber\\
  &=
  \bar Z_Q(s).
\end{align}
Therefore Theorem~1 gives unrestricted value $\alpha\log \bar Z_Q(s)$, while Theorem~\ref{thm:fixed_base_optimum} gives fixed-base value
\begin{equation}
  \alpha
  \mathbb E_{z_0\sim p_0}
  \left[
    \log Z_Q(s,z_0)
  \right]
  -
  \alpha
  D_{\mathrm{KL}}
  \left(
    p_0
    \,\|\,
    r_0
  \right).
\end{equation}
Their difference is Eq.~\eqref{eq:base_jensen_gap}. This difference is nonnegative because it is the initial-law KL to the unrestricted tilted bridge. Indeed, the unrestricted tilted bridge has initial marginal
\begin{equation}
  P_{Q,0}^\star(dz_0\mid s)
  =
  \frac{
    Z_Q(s,z_0)
  }{
    \bar Z_Q(s)
  }
  r_0(dz_0).
\end{equation}
Thus
\begin{align}
  D_{\mathrm{KL}}
  \left(
    p_0
    \,\|\,
    P_{Q,0}^\star(\cdot\mid s)
  \right)
  &=
  \mathbb E_{z_0\sim p_0}
  \left[
    \log
    \frac{
      dp_0
    }{
      dr_0
    }(z_0)
    +
    \log \bar Z_Q(s)
    -
    \log Z_Q(s,z_0)
  \right]
  \nonumber\\
  &=
  \frac{
    \Delta(s)
  }{
    \alpha
  }.
\end{align}
This proves nonnegativity and the equality condition. If $p_0=r_0$, the condition reduces to $Z_Q(s,z_0)$ being constant under $p_0$.

It remains to relate this gap to the path and endpoint laws. The unrestricted tilted bridge satisfies
\begin{equation}
  \frac{
    dP_Q^\star
  }{
    dR
  }(\tau\mid s)
  =
  \frac{
    \exp(Q(s,T(z_K))/\alpha)
  }{
    \bar Z_Q(s)
  },
\end{equation}
while the fixed-base optimum satisfies
\begin{equation}
  \frac{
    dP_{Q,p_0}^\star
  }{
    dR
  }(\tau\mid s)
  =
  \frac{
    dp_0
  }{
    dr_0
  }(z_0)
  \frac{
    \exp(Q(s,T(z_K))/\alpha)
  }{
    Z_Q(s,z_0)
  }.
\end{equation}
Therefore
\begin{align}
  D_{\mathrm{KL}}
  \left(
    P_{Q,p_0}^\star
    \,\|\,P_Q^\star
  \right)
  &=
  \mathbb E_{z_0\sim p_0}
  \left[
    \log
    \frac{
      dp_0
    }{
      dr_0
    }(z_0)
    +
    \log
    \frac{
      \bar Z_Q(s)
    }{
      Z_Q(s,z_0)
    }
  \right]
  \nonumber\\
  &=
  \frac{
    \Delta(s)
  }{
    \alpha
  }.
\end{align}
The expectation above is under $P_{Q,p_0}^\star$. Its initial marginal is still $p_0$ by construction, which is why the expectation reduces to $z_0\sim p_0$. Finally, the endpoint action is the measurable map $\tau\mapsto T(z_K)$. Data processing for relative entropy under this map gives
\begin{equation}
  D_{\mathrm{KL}}
  \left(
    \pi_{Q,p_0}^\star(\cdot\mid s)
    \,\|\,\pi_Q^\star(\cdot\mid s)
  \right)
  \le
  D_{\mathrm{KL}}
  \left(
    P_{Q,p_0}^\star(\cdot\mid s)
    \,\|\,P_Q^\star(\cdot\mid s)
  \right),
\end{equation}
which proves Eq.~\eqref{eq:shared_base_endpoint_bound}.
\end{proof}

\subsection{Proof of Lemma~\ref{lem:control_energy}}
\label{app:proof_control_energy}

\paragraph{Restatement.}
Let
\[
  P_\theta(d\tau\mid s)
  =
  p_0(dz_0)
  \prod_{k=0}^{K-1}
  q_{\theta,k}(dz_{k+1}\mid z_k,s),
  \qquad
  R(d\tau)
  =
  r_0(dz_0)
  \prod_{k=0}^{K-1}
  r_k(dz_{k+1}\mid z_k).
\]
Then
\[
  D_{\mathrm{KL}}
  \left(
    P_\theta(\cdot\mid s)
    \,\|\,R
  \right)
  =
  D_{\mathrm{KL}}
  \left(
    p_0
    \,\|\,
    r_0
  \right)
  +
  \sum_{k=0}^{K-1}
  \mathbb E_{\tau\sim P_\theta}
  \left[
    D_{\mathrm{KL}}
    \left(
      q_{\theta,k}(\cdot\mid z_k,s)
      \,\|\,r_k(\cdot\mid z_k)
    \right)
  \right].
\]
For Gaussian local kernels, if $\mathcal C_\theta(s,\tau)$ is the sum of the local Gaussian KL terms, then
\[
  D_{\mathrm{KL}}
  \left(
    P_\theta(\cdot\mid s)
    \,\|\,R
  \right)
  =
  D_{\mathrm{KL}}
  \left(
    p_0
    \,\|\,
    r_0
  \right)
  +
  \mathbb E_{\tau\sim P_\theta(\cdot\mid s)}
  \left[
    \mathcal C_\theta(s,\tau)
  \right].
\]

\begin{proof}
The path likelihood ratio separates into an initial-base term and local transition terms:
\begin{equation}
  \log
  \frac{
    dP_\theta(\tau\mid s)
  }{
    dR(\tau)
  }
  =
  \log
  \frac{
    dp_0
  }{
    dr_0
  }(z_0)
  +
  \sum_{k=0}^{K-1}
  \log
  \frac{
    q_{\theta,k}(z_{k+1}\mid z_k,s)
  }{
    r_k(z_{k+1}\mid z_k)
  }.
\end{equation}
Taking expectation under the actor path law gives the initial contribution
\begin{equation}
  \mathbb E_{z_0\sim p_0}
  \left[
    \log
    \frac{
      dp_0
    }{
      dr_0
    }(z_0)
  \right]
  =
  D_{\mathrm{KL}}
  \left(
    p_0
    \,\|\,
    r_0
  \right).
\end{equation}
\begin{align}
  \mathbb E_{\tau\sim P_\theta(\cdot\mid s)}
  \left[
    \log
    \frac{
      q_{\theta,k}(z_{k+1}\mid z_k,s)
    }{
      r_k(z_{k+1}\mid z_k)
    }
  \right]
  &=
  \mathbb E_{z_k\sim P_{\theta,k}(\cdot\mid s)}
  \mathbb E_{z_{k+1}\sim q_{\theta,k}(\cdot\mid z_k,s)}
  \left[
    \log
    \frac{
      q_{\theta,k}(z_{k+1}\mid z_k,s)
    }{
      r_k(z_{k+1}\mid z_k)
    }
  \right]
  \nonumber\\
  &=
  \mathbb E_{z_k\sim P_{\theta,k}(\cdot\mid s)}
  \left[
    D_{\mathrm{KL}}
    \left(
      q_{\theta,k}(\cdot\mid z_k,s)
      \,\|\,r_k(\cdot\mid z_k)
    \right)
  \right],
\end{align}
where $P_{\theta,k}(\cdot\mid s)$ is the actor marginal law of $z_k$. Summing this identity over $k$ gives Eq.~\eqref{eq:local_kl_sum}. If the local kernels are Gaussian, each local KL has a closed form. Defining $\mathcal C_\theta(s,\tau)$ as the sum of those local Gaussian KLs gives Eq.~\eqref{eq:path_kl_control_energy}.
\end{proof}

\subsection{Proof of Proposition~\ref{prop:actor_reverse_kl}}
\label{app:proof_actor_projection}

\paragraph{Restatement.}
Let $P_Q^\star$ be the unrestricted soft-optimal bridge in Theorem~1. For any finite-step bridge actor with fixed base law $p_0\ll r_0$ satisfying Lemma~\ref{lem:control_energy},
\[
  \mathcal L_{\mathrm{actor}}(\theta\mid s)
  =
  \mathbb E_{\tau\sim P_\theta(\cdot\mid s)}
  \left[
    \alpha\,\mathcal C_\theta(s,\tau)
    -
    Q(s,T(z_K))
  \right]
\]
satisfies
\[
  \mathcal L_{\mathrm{actor}}(\theta\mid s)
  =
  \alpha
  D_{\mathrm{KL}}
  \left(
    P_\theta(\cdot\mid s)
    \,\|\,P_Q^\star(\cdot\mid s)
  \right)
  -
  \alpha\log Z_Q(s)
  -
  \alpha
  D_{\mathrm{KL}}
  \left(
    p_0
    \,\|\,
    r_0
  \right).
\]
Consequently, minimizing this loss over the unrestricted fixed-base path-law family has global minimizer $P_{Q,p_0}^\star$.

\begin{proof}
Fix a state $s$. By Lemma~\ref{lem:control_energy}, the sampled control energy satisfies
\begin{equation}
  \mathbb E_{\tau\sim P_\theta(\cdot\mid s)}
  \left[
    \mathcal C_\theta(s,\tau)
  \right]
  =
  D_{\mathrm{KL}}
  \left(
    P_\theta(\cdot\mid s)
    \,\|\,R
  \right)
  -
  D_{\mathrm{KL}}
  \left(
    p_0
    \,\|\,
    r_0
  \right).
\end{equation}
The unrestricted soft-optimal bridge from Theorem~1 has likelihood ratio
\begin{equation}
  \log
  \frac{
    dP_Q^\star
  }{
    dR
  }(\tau\mid s)
  =
  \frac{
    Q(s,T(z_K))
  }{
    \alpha
  }
  -
  \log Z_Q(s).
\end{equation}
Expanding the reverse KL to this ideal bridge gives
\begin{align}
  &\alpha
  D_{\mathrm{KL}}
  \left(
    P_\theta(\cdot\mid s)
    \,\|\,P_Q^\star(\cdot\mid s)
  \right)
  \nonumber\\
  &=
  \alpha
  \mathbb E_{\tau\sim P_\theta(\cdot\mid s)}
  \left[
    \log
    \frac{
      dP_\theta
    }{
      dR
    }
    (\tau\mid s)
    -
    \frac{
      Q(s,T(z_K))
    }{
      \alpha
    }
    +
    \log Z_Q(s)
  \right]
  \nonumber\\
  &=
  \alpha
  D_{\mathrm{KL}}
  \left(
    P_\theta(\cdot\mid s)
    \,\|\,R
  \right)
  -
  \mathbb E_{\tau\sim P_\theta}
  \left[
    Q(s,T(z_K))
  \right]
  +
  \alpha\log Z_Q(s)
  \nonumber\\
  &=
  \mathcal L_{\mathrm{actor}}(\theta\mid s)
  +
  \alpha\log Z_Q(s)
  +
  \alpha
  D_{\mathrm{KL}}
  \left(
    p_0
    \,\|\,
    r_0
  \right).
\end{align}
Rearranging gives Eq.~\eqref{eq:actor_reverse_kl}. Since the projection is optimized over the unrestricted fixed-base family, Theorem~\ref{thm:fixed_base_optimum} identifies the global minimizer as $P_{Q,p_0}^\star$. The finite neural Gaussian Markov actor used by the algorithm is a restricted parameterization of this family, so the proof gives the population projection target rather than a global optimization guarantee for SGD.
\end{proof}

\section{Algorithm Pseudocode}
\label{app:pseudocode}

\begin{algorithm}[H]
\caption{\myalgo off-policy training}
\label{alg:softgac}
\begin{algorithmic}[1]
\Require replay buffer $\mathcal D$, actor $P_\theta$, target actor $P_{\bar\theta}$, critic $Q_\phi$, target critic $Q_{\bar\phi}$, temperature $\alpha$, target cost $\mathcal C_{\mathrm{target}}$, discount $\gamma$, Polyak coefficient $\rho$
\State initialize $\mathcal D$, $\theta$, $\phi$, $\bar\theta\leftarrow\theta$, $\bar\phi\leftarrow\phi$
\For{each environment step}
  \State sample a bridge path $\tau\sim P_\theta(\cdot\mid s)$ and execute $a=T(z_K)$
  \State store $(s,a,r,s',d)$ in $\mathcal D$
  \For{each gradient update}
    \State sample a replay batch $(s,a,r,s',d)\sim\mathcal D$
    \State sample next paths $\tau'\sim P_{\bar\theta}(\cdot\mid s')$
    \State compute next actions $a'=T(z'_K)$ and bridge costs $\mathcal C_{\bar\theta}(s',\tau')$
    \State build soft targets $y=r+\gamma(1-d)\left[Q_{\bar\phi}^{\min}(s',a')-\alpha\mathcal C_{\bar\theta}(s',\tau')\right]$
    \State update $Q_\phi$ toward $y$ with the chosen off-policy critic loss
    \If{the policy-delay interval is reached}
      \State sample current paths $\tau\sim P_\theta(\cdot\mid s)$
      \State compute current actions $\tilde a=T(z_K)$ and bridge costs $\mathcal C_\theta(s,\tau)$
      \State update $\theta$ to minimize $\mathbb E[\alpha\mathcal C_\theta(s,\tau)-Q_\phi^{\min}(s,\tilde a)]$
      \State update $\log\alpha$ to minimize $\mathbb E[\alpha(\mathcal C_{\mathrm{target}}-\mathcal C_\theta(s,\tau))]$ with $\alpha=\exp(\log\alpha)$
    \EndIf
    \If{using target networks and the target-update interval is reached}
      \State soft-update target networks $\bar\theta\leftarrow\rho\bar\theta+(1-\rho)\theta$ and $\bar\phi\leftarrow\rho\bar\phi+(1-\rho)\phi$
    \EndIf
  \EndFor
\EndFor
\end{algorithmic}
\end{algorithm}

\section{2D Bridge Visualization}
\label{app:controlled_2d}

Figure~\ref{fig:toy_bridge_2d} visualizes the bridge mechanism in a controlled two-dimensional fixed-critic setting. This diagnostic setting isolates the actor update by training a bridge policy against a hand-designed multimodal value function. It makes the finite-step path distribution visible as it moves from the base law to the terminal action law under different control-energy budgets.

\begin{figure}[t]
  \centering
  \includegraphics[width=\linewidth]{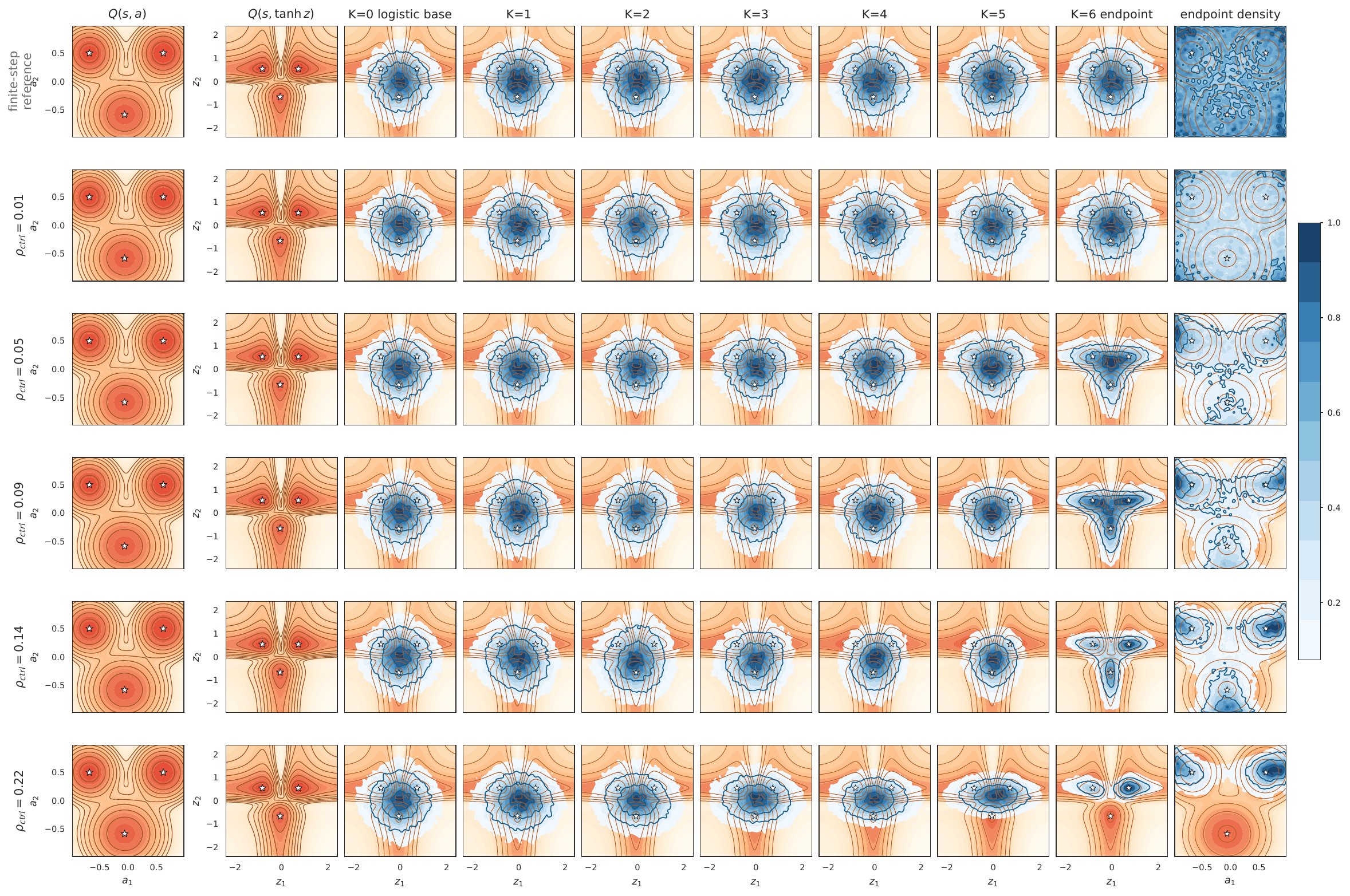}
  \caption{2D bridge visualization. Each row corresponds to a different control-energy budget. Columns show the action-space critic, the corresponding pre-tanh latent critic, intermediate latent bridge densities and the final action-space policy density.}
  \label{fig:toy_bridge_2d}
\end{figure}

The bounded action is $a=(a_1,a_2)\in(-1,1)^2$. We define an unnormalized fixed critic as a log-sum-exp mixture of three Gaussian-like modes,
\begin{equation}
  Q_{\mathrm{raw}}(a)
  =
  \log
  \sum_{m=1}^{3}
  w_m
  \exp
  \left(
    -
    \frac{1}{2}
    \left\|
      \frac{a-c_m}{\sigma_m}
    \right\|_2^2
  \right),
\end{equation}
with centers
\begin{equation}
  c_1=(-0.68,0.50),
  \qquad
  c_2=(0.62,0.50),
  \qquad
  c_3=(-0.06,-0.58),
\end{equation}
axis-wise widths
\begin{equation}
  \sigma_1=(0.24,0.24),
  \qquad
  \sigma_2=(0.24,0.24),
  \qquad
  \sigma_3=(0.34,0.32),
\end{equation}
and weights $(w_1,w_2,w_3)=(0.90,1.00,0.70)$. We normalize the critic to $[0,1]$ over a dense action grid and use this normalized value as $Q(s,a)$ for a single dummy state $s$. The latent landscape shown in the second column is the pullback $Q(s,\tanh z)$.

The reference process is the same finite-step reference used in the method section. We sample the base action from an approximately uniform law on $[-0.995,0.995]^2$ and set $z_0=\mathrm{arctanh}(a_0)$. The reference latent process uses $K=6$ Euler steps with $h=1/K$,
\begin{equation}
  z_{k+1}
  =
  z_k
  -
  2h\tanh(z_k)
  +
  \sqrt{2h}\,\epsilon_k,
  \qquad
  \epsilon_k\sim\mathcal N(0,I).
\end{equation}
This finite chain approximates the continuous reference whose stationary action law after $\tanh$ is uniform. We intentionally show the finite-step reference rather than the continuous-limit uniform distribution because the algorithm optimizes the finite-step path KL.

For each control-energy budget, we train a separate bridge actor with $K=6$ action-dimensional latent transitions. Given a sampled logistic reference base latent and Gaussian transition noise, the actor produces a path $\tau=(z_0,\ldots,z_K)$ and terminal action $a=\tanh(z_K)$. We optimize the sampled actor objective
\begin{equation}
  \mathbb E_{\tau\sim P_\theta}
  \left[
    \alpha\,\mathcal C_\theta(s,\tau)
    -
    Q(s,\tanh z_K)
  \right],
\end{equation}
where $\mathcal C_\theta$ is the finite-step local Gaussian control energy from Lemma~\ref{lem:control_energy}. The temperature is tuned with the same budget form as the main algorithm,
\begin{equation}
  \mathcal C_{\mathrm{target}}
  =
  \rho_{\mathrm{ctrl}} K d_a,
  \qquad d_a=2.
\end{equation}
The rows in Figure~\ref{fig:toy_bridge_2d} use
\begin{equation}
  \rho_{\mathrm{ctrl}}\in\{0.01,0.05,0.09,0.14,0.22\}.
\end{equation}
Small budgets keep the actor close to the reference bridge and preserve broad coverage. Larger budgets allow stronger value guidance and concentrate the endpoint density near higher-value modes. The first two columns show $Q(s,a)$ in action space and $Q(s,\tanh z)$ in pre-tanh latent space. The next columns show latent densities at bridge steps $K=0,\ldots,6$. The final column maps terminal latents back to action space and overlays the resulting policy density on the action-space critic.

\section{Finite-Step Reference Endpoint Bias}
\label{app:finite_reference_bias}

In this paper, we use a finite-step reference bridge to approximate the ideal continuous reference whose endpoint action law is uniform. This approximation affects only the endpoint marginal law used as the reference prior. It is distinct from the path-wise KL used by the actor update, which remains exact for the chosen finite-step actor-reference pair. We quantify this endpoint effect in Figure~\ref{fig:reference_endpoint_bias}.

\begin{figure}[t]
  \centering
  \begin{minipage}[t]{0.48\linewidth}
    \centering
    \includegraphics[width=\linewidth]{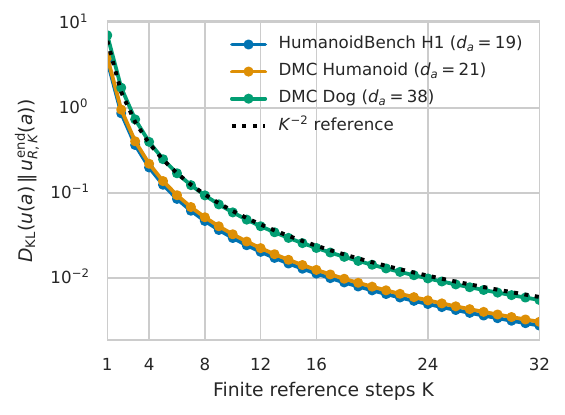}
  \end{minipage}
  \hfill
  \begin{minipage}[t]{0.48\linewidth}
    \centering
    \includegraphics[width=\linewidth]{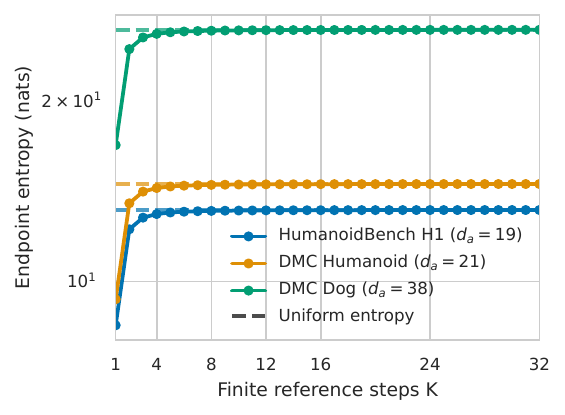}
  \end{minipage}
  \caption{Finite-step reference endpoint bias. Left shows the endpoint action-marginal KL from the ideal uniform prior to the finite-step reference endpoint law. Right shows the finite-step endpoint entropy and the corresponding uniform entropy baselines. The three curves use the action dimensions of HumanoidBench H1, DMC Humanoid and DMC Dog.}
  \label{fig:reference_endpoint_bias}
\end{figure}

Let $q_{\mathrm{ref}}(z)=\frac{1}{2}\mathrm{sech}^2(z)$ be the one-dimensional pre-tanh density whose image under $\tanh$ is uniform. For a reference bridge with $K$ Euler steps, let $p_K$ be the one-dimensional terminal latent density produced by
\begin{equation}
  z_{k+1}
  =
  z_k
  -
  2h\tanh(z_k)
  +
  \sqrt{2h}\,\epsilon_k,
  \qquad h=\frac{1}{K}.
\end{equation}
The map $\tanh$ is bijective between latent space and the normalized action interval, so endpoint KL values can be computed in latent space without estimating an action-space density. For example,
\begin{equation}
  D_{\mathrm{KL}}\!\left(u(a)\,\middle\|\,u^{\mathrm{end}}_{R,K}(a)\right)
  =
  D_{\mathrm{KL}}\!\left(q_{\mathrm{ref}}(z)\,\middle\|\,p_K(z)\right).
\end{equation}
The implemented reference factorizes across action dimensions, so the $d_a$-dimensional laws are product measures and relative entropy is additive across coordinates. It is therefore enough to compute the one-dimensional terminal density and multiply the resulting KL by $d_a$. Let $\mathcal T_h$ be the one-dimensional Euler transition operator,
\begin{equation}
  (\mathcal T_h f)(z')
  =
  \int
  \mathcal N\!\left(
    z'
    \,\middle|\,
    z-2h\tanh z,
    2h
  \right)
  f(z)\,dz,
  \qquad h=\frac{1}{K}.
\end{equation}
Then $p_K=\mathcal T_h^K q_{\mathrm{ref}}$, and the $d_a$-dimensional endpoint-prior gap has the exact finite-$K$ form
\begin{equation}
  G_K(d_a)
  =
  D_{\mathrm{KL}}\!\left(u(a)\,\middle\|\,u^{\mathrm{end}}_{R,K}(a)\right)
  =
  d_a
  \int
  q_{\mathrm{ref}}(z)
  \log
  \frac{q_{\mathrm{ref}}(z)}{p_K(z)}
  dz.
  \label{eq:finite_reference_endpoint_gap}
\end{equation}
Thus the dependence on $d_a$ is exactly linear. The dependence on $K$ comes only from the Euler discretization error in $p_K$. Under the standard first-order density expansion for Euler discretization over a fixed time horizon, $p_K(z)=q_{\mathrm{ref}}(z)(1+K^{-1}r(z)+O(K^{-2}))$ with $\int q_{\mathrm{ref}}(z)r(z)\,dz=0$. The first-order term cancels in the KL because of normalization, so the leading contribution is quadratic in the density error. Substituting this expansion into Eq.~\eqref{eq:finite_reference_endpoint_gap} gives
\begin{equation}
  G_K(d_a)
  =
  \frac{d_a}{2K^2}
  \int
  q_{\mathrm{ref}}(z)r(z)^2\,dz
  +
  O\!\left(\frac{d_a}{K^3}\right).
  \label{eq:finite_reference_gap_rate}
\end{equation}
Thus the leading endpoint-prior gap scales as $O(d_a/K^2)$ whenever this expansion holds. This rate explains the observed finite-step bias but is not an assumption used by the algorithm. The entropy deficit has the same leading rate because it equals $D_{\mathrm{KL}}(u^{\mathrm{end}}_{R,K}\|u)$.
Likewise, the endpoint entropy satisfies
\begin{equation}
  H\!\left(u^{\mathrm{end}}_{R,K}\right)
  =
  d_a\log 2
  -
  D_{\mathrm{KL}}\!\left(u^{\mathrm{end}}_{R,K}\,\middle\|\,u\right),
\end{equation}
where $d_a\log 2$ is the $d_a$-dimensional uniform endpoint entropy. We compute the one-dimensional density deterministically on a dense grid by applying the Gaussian transition kernel, then report the corresponding $d_a$-dimensional quantities for the action dimensions used in our experiments.

The bias decreases quickly with $K$. At the default $K=6$, the endpoint KL $D_{\mathrm{KL}}(u\,\|\,u^{\mathrm{end}}_{R,K})$ is approximately $0.085$, $0.094$ and $0.169$ nats for $d_a=19$, $21$ and $38$. The endpoint entropy is already close to the uniform entropy in all three domains, while the network is still shallow to optimize. This supports the interpretation that the finite reference mainly provides an implementable high-entropy prior, while the algorithm optimizes the exact finite-step path regularizer defined by that prior.

\section{Detailed Experimental Setup}
\label{app:experimental_setup}

\subsection{Environment and Task Domains}
\label{app:env_task_domains}

Our benchmark contains $12$ continuous-control tasks from three domains. Figure~\ref{fig:benchmark_task_screenshots} shows representative rendered states. All experiments use vector observations and continuous bounded actions. The reported dimensions are the observation and action dimensions after the environment wrappers used by our training code.

\begin{figure}[H]
  \centering
  \includegraphics[width=\linewidth]{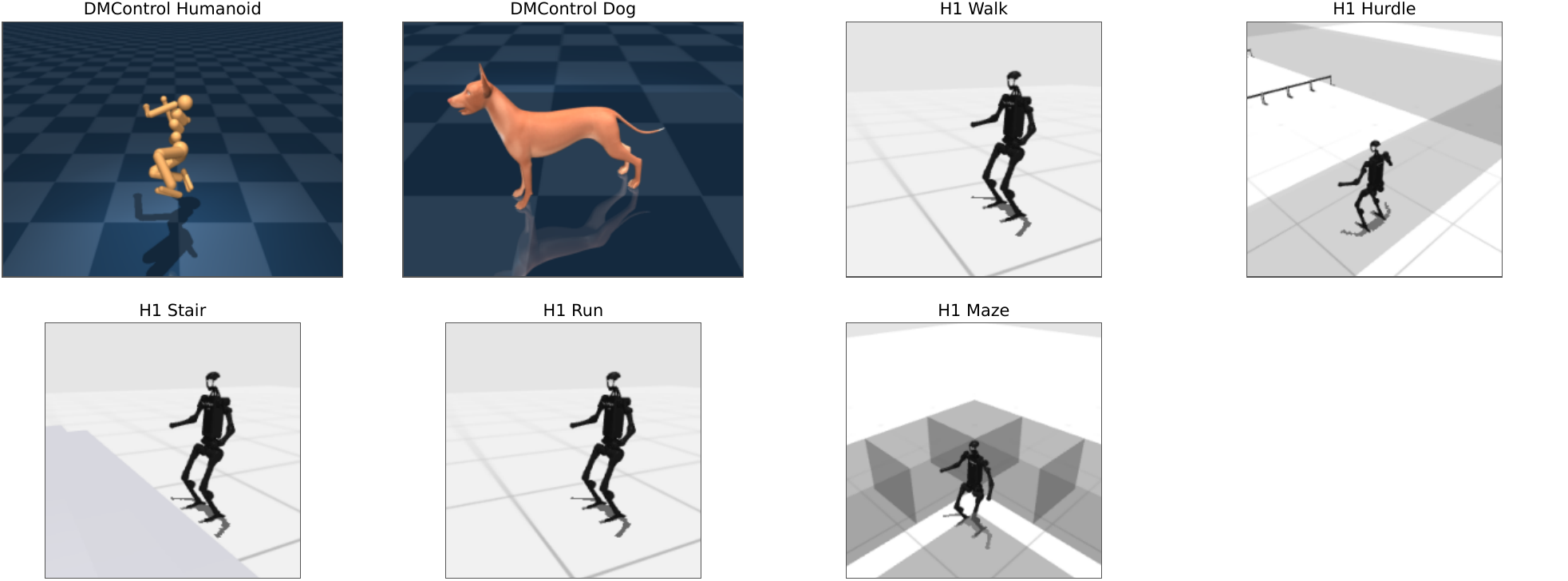}
  \caption{Representative screenshots of the benchmark domains and H1 tasks.}
  \label{fig:benchmark_task_screenshots}
\end{figure}

\begin{table}[H]
  \centering
  \footnotesize
  \setlength{\tabcolsep}{5.0pt}
  \caption{Benchmark task dimensions.}
  \label{tab:env_task_dims}
  \begin{tabular}{llcc}
    \toprule
    Domain & Task & Observation dim. & Action dim. \\
    \midrule
    DMC Humanoid & \texttt{dm\_control/humanoid-run} & 67 & 21 \\
    DMC Humanoid & \texttt{dm\_control/humanoid-walk} & 67 & 21 \\
    DMC Humanoid & \texttt{dm\_control/humanoid-stand} & 67 & 21 \\
    DMC Dog & \texttt{dm\_control/dog-run} & 223 & 38 \\
    DMC Dog & \texttt{dm\_control/dog-trot} & 223 & 38 \\
    DMC Dog & \texttt{dm\_control/dog-walk} & 223 & 38 \\
    DMC Dog & \texttt{dm\_control/dog-stand} & 223 & 38 \\
    HumanoidBench H1 & \texttt{h1-walk-v0} & 51 & 19 \\
    HumanoidBench H1 & \texttt{h1-hurdle-v0} & 51 & 19 \\
    HumanoidBench H1 & \texttt{h1-stair-v0} & 51 & 19 \\
    HumanoidBench H1 & \texttt{h1-run-v0} & 51 & 19 \\
    HumanoidBench H1 & \texttt{h1-maze-v0} & 51 & 19 \\
    \bottomrule
  \end{tabular}
\end{table}

DMC Humanoid evaluates high-dimensional biped locomotion with running, walking and standing objectives. DMC Dog evaluates quadruped control with a larger observation and action space, including fast running, trotting, walking and standing. HumanoidBench H1 uses a humanoid robot with $19$ actuated degrees of freedom. The walk and run tasks measure whole-body locomotion, while hurdle, stair and maze add obstacle negotiation, contact-rich terrain interaction and navigation constraints.

\paragraph{Benchmark assets and licenses.}
We use the DeepMind Control Suite through \texttt{dm\_control} version 1.0.39 from \url{https://github.com/google-deepmind/dm_control}, which is released under the Apache License 2.0. We use HumanoidBench through the official repository's main branch from \url{https://github.com/carlosferrazza/humanoid-bench}, which is released under the MIT License and includes third-party notices in its repository license file.

\paragraph{Baseline implementation references.}
All baseline results are produced by our unified JAX codebase. We use the official repositories in Table~\ref{tab:baseline_code_references} as implementation references, then re-implement the algorithm-specific actor, target construction and auxiliary losses in JAX inside our training stack. This gives each method the same logging, replay, evaluation and experiment infrastructure while preserving the algorithm-specific components required by the original implementations.

\begin{table}[H]
  \centering
  \footnotesize
  \setlength{\tabcolsep}{5pt}
  \caption{Official implementation repositories used as baseline references.}
  \label{tab:baseline_code_references}
  \begin{tabular}{ll}
    \toprule
    Method & Official repository \\
    \midrule
    FLAC & \url{https://github.com/bytedance/FLAC} \\
    DIME & \url{https://github.com/ALRhub/DIME} \\
    FlowRL & \url{https://github.com/bytedance/FlowRL} \\
    QSM & \url{https://github.com/escontra/score_matching_rl} \\
    QVPO & \url{https://github.com/wadx2019/qvpo} \\
    CrossQ-SAC & \url{https://github.com/adityab/CrossQ} \\
    DACER & \url{https://github.com/happy-yan/DACER-Diffusion-with-Online-RL} \\
    MaxEntDP & \url{https://github.com/diffusionyes/MaxEntDP} \\
    SDAC & \url{https://github.com/mahaitongdae/diffusion_policy_online_rl} \\
    Normalizing Flow & \url{https://github.com/Princeton-RL/normalising-flows-4reinforcement-learning} \\
    \bottomrule
  \end{tabular}
\end{table}

\subsection{Training Hardware}
\label{app:training_hardware}

We train agents on Google Cloud TPU v4-16 and v6e-8 VMs, running independent seeds in parallel with one TPU core per seed. On the v6e hosts, Linux reports two AMD EPYC 9B14 sockets with $90$ cores per socket, $180$ logical CPUs, one thread per core and approximately $1.4$ TiB of system memory. The machines run Ubuntu Linux with kernel 6.8 on x86\_64 CPUs. Appendix~\ref{app:more_results} reports training wall-clock curves, and Appendix~\ref{app:additional_compute} specifies the hardware used for inference and training-time comparisons.

\subsection{Implementation Details}
\label{app:implementation_details}

The \myalgo actor operates in pre-tanh latent space and samples a tensor of noises with shape $B\times(K+1)\times d_a$. The first slice gives the base latent $z_0$. By default, we use the logistic base by sampling $a_0\sim\mathrm{Unif}((-1,1)^{d_a})$ and setting $z_0=\mathrm{arctanh}(a_0)$. The actor forward pass is:

\begin{verbatim}
def bridge_actor(obs, noise):
    z = noise[:, 0]                         # base latent z_0
    latents, drifts, sigmas = [z], [], []
    h = 1.0 / K
    for k in range(K):
        x = concat(obs, z)
        x = LayerNorm(x)
        x = Dense(hidden_size)(x)
        x = elu(x)
        x = LayerNorm(x)
        drift = Dense(action_dim)(x)
        sigma = softplus(Dense(action_dim)(x))
        z = z + h * drift + sqrt(2 * h) * sigma * noise[:, k + 1]
        latents.append(z)
        drifts.append(drift)
        sigmas.append(sigma)
    action = action_scale * tanh(z) + action_bias
    return action, latents, drifts, sigmas
\end{verbatim}

Each block uses one hidden layer with layer normalization before and after the ELU activation, followed by drift and positive diagonal-scale heads. Main runs use $K=6$ blocks and width $512$, with width $256$ for DMC Dog to keep the actor parameter budget close to the baselines. The actor update minimizes $\alpha\mathcal C_\theta(s,\tau)-Q_\phi^{\min}(s,a)$. The control energy $\mathcal C_\theta$ is the accumulated closed-form Gaussian transition KL to the reference bridge, so no endpoint entropy estimator is used. The temperature $\alpha=\exp(\log\alpha)$ is tuned toward $\rho_{\mathrm{ctrl}}Kd_a$, with default $\rho_{\mathrm{ctrl}}=0.2$. The main critic is the same twin C51 critic with CrossQ-style update adopted in DIME~\citep{DIME25}.

\subsection{Choosing the Control-Energy Budget}
\label{app:rho_ctrl_choice}

The target control-energy budget is meant to set how strongly the value term can move the bridge away from the high-entropy reference, not to introduce a task-specific objective. A useful way to choose its scale is to ask how far one action dimension should be allowed to move under value-guided control. The path KL factorizes over $K$ transitions and $d_a$ action dimensions, so we set
\begin{equation}
  \mathcal C_{\mathrm{target}}
  =
  \rho_{\mathrm{ctrl}} K d_a.
\end{equation}
This makes $\rho_{\mathrm{ctrl}}$ the average local KL budget, measured in nats, per transition and per action dimension. Equivalently, each action dimension receives a total path budget of $\rho_{\mathrm{ctrl}}K$ along the full bridge.

This scale has a simple endpoint interpretation. Ignore the learned variance term for a moment and assume that the actor and reference transitions share covariance $2hI$, where $h=1/K$. If the actor adds a one-dimensional mean control $\delta z_k$ at step $k$, the local KL contribution is
\begin{equation}
  D_{\mathrm{KL}}
  \left(
    \mathcal N(\mu_R+\delta z_k,2h)
    \,\|\,
    \mathcal N(\mu_R,2h)
  \right)
  =
  \frac{(\delta z_k)^2}{4h}.
\end{equation}
For a desired terminal displacement $\Delta z$ in one pre-tanh action dimension, the cheapest constant-speed control has $\delta z_k\approx \Delta z/K$. The accumulated control energy is then
\begin{equation}
  \sum_{k=0}^{K-1}
  \frac{(\Delta z/K)^2}{4h}
  =
  \frac{(\Delta z)^2}{4}.
\end{equation}
Thus a budget $\rho_{\mathrm{ctrl}}K$ per action dimension roughly permits a terminal pre-tanh displacement
\begin{equation}
  |\Delta z|
  \approx
  2\sqrt{\rho_{\mathrm{ctrl}}K}.
\end{equation}
If we want the actor to be able to push a dimension close to the saturated action region, say $|a|\approx 0.95$ to $0.98$, the corresponding pre-tanh scale is $z_{\mathrm{sat}}=\mathrm{arctanh}(|a|)\approx 1.8$ to $2.3$. Matching $\rho_{\mathrm{ctrl}}K\approx z_{\mathrm{sat}}^2/4$ gives $\rho_{\mathrm{ctrl}}\approx z_{\mathrm{sat}}^2/(4K)$. For the main bridge depth $K=6$, this range is about $0.14$ to $0.22$. We therefore use $\rho_{\mathrm{ctrl}}=0.2$ as a moderate default. It allows value-guided transport toward near-saturated high-value actions, but still assigns a visible cost to collapse away from the broad reference. Smaller budgets keep broader action coverage, while larger budgets permit sharper value-guided concentration. This is the same qualitative trend visualized in Appendix~\ref{app:controlled_2d}.

\subsection{Hyperparameters}
\label{app:hyperparameters}

Table~\ref{tab:actor_parameter_counts_main} reports the actor and main critic parameter counts by domain, and Table~\ref{tab:main_hyperparameters} summarizes the main hyperparameters. All methods use $8$ seeds and Adam optimizers. To reduce critic-side confounding, we use the same twin C51 main critic and CrossQ-style main critic update adopted in DIME~\cite{DIME25} across the compared methods. When an algorithm includes additional auxiliary networks, such as FlowRL's buffer critic, we keep the corresponding architecture and update rule from the official implementation.

\begin{table}[H]
  \centering
  \setlength{\tabcolsep}{4pt}
  \caption{Actor and main critic parameter counts in the experiment domains, in millions of parameters.}
  \label{tab:actor_parameter_counts_main}
  \begin{tabular}{lcccccc}
    \toprule
    & \multicolumn{2}{c}{DMC Humanoid} & \multicolumn{2}{c}{DMC Dog} & \multicolumn{2}{c}{HumanoidBench} \\
    \cmidrule(lr){2-3}\cmidrule(lr){4-5}\cmidrule(lr){6-7}
    Algorithm & Actor & Critic & Actor & Critic & Actor & Critic \\
    \midrule
    \myalgo & 0.410 & 9.19 & 0.526 & 9.90 & 0.342 & 9.11 \\
    FLAC & 0.322 & 9.19 & 0.419 & 9.90 & 0.311 & 9.11 \\
    DIME & 0.423 & 9.19 & 0.472 & 9.90 & 0.418 & 9.11 \\
    FlowRL & 0.322 & 9.19 & 0.419 & 9.90 & 0.311 & 9.11 \\
    QSM & 0.614 & 9.19 & 0.712 & 9.90 & 0.604 & 9.11 \\
    QVPO & 0.369 & 9.19 & 0.466 & 9.90 & 0.359 & 9.11 \\
    CrossQ-SAC & 0.321 & 9.19 & 0.419 & 9.90 & 0.311 & 9.11 \\
    \bottomrule
  \end{tabular}
\end{table}

\begin{table}[H]
  \centering
  \caption{Main hyperparameters used in the experiments.}
  \label{tab:main_hyperparameters}
  \setlength{\tabcolsep}{3pt}
  \renewcommand{\arraystretch}{1.12}
  \resizebox{\linewidth}{!}{%
  \begin{tabular}{lccccccc}
    \toprule
    Hyperparameter & \myalgo & FLAC & DIME & FlowRL & QSM & QVPO & CrossQ-SAC \\
    \midrule
    UTD ratio & 2 & 2 & 2 & 2 & 2 & 2 & 2 \\
    Discount & 0.99 & 0.99 & 0.99 & 0.99 & 0.99 & 0.99 & 0.99 \\
    Batch size & 256 & 256 & 256 & 256 & 256 & 256 & 256 \\
    Buffer size & $10^6$ & $10^6$ & $10^6$ & $10^6$ & $10^6$ & $10^6$ & $10^6$ \\
    Learn starts & 5000 & 5000 & 5000 & 5000 & 5000 & 5000 & 5000 \\
    Actor lr & \num{3e-4} & \num{3e-4} & \num{3e-4} & \num{3e-4} & \num{3e-4} & \num{3e-4} & \num{3e-4} \\
    Critic lr & \num{3e-4} & \num{3e-4} & \num{3e-4} & \num{3e-4} & \num{3e-4} & \num{3e-4} & \num{3e-4} \\
    Policy delay & 2 & 2 & 2 & 2 & 2 & 2 & 2 \\
    Regularizer & path KL & kinetic & entropy bd. & CFM & score-Q & weighted-Q & entropy \\
    Target entropy/energy & $0.2Kd_a$ & $0.5d_a$ & $4d_a$ & N/A & N/A & N/A & $-d_a$ \\
    Temp. lr & $10^{-3}$ & $3\cdot10^{-5}$ & $10^{-3}$ & N/A & N/A & N/A & $10^{-3}$ \\
    Critic depth & 2 & 2 & 2 & 2 & 2 & 2 & 2 \\
    Critic hidden size & 2048 & 2048 & 2048 & 2048 & 2048 & 2048 & 2048 \\
    Num. bins & 101 & 101 & 101 & 101 & 101 & 101 & 101 \\
    Actor depth & 6 & 2 & 3 & 2 & 2 & 2 & 2 \\
    Actor width & 256/512\textsuperscript{\dag} & 512 & 256 & 512 & 512 & 512 & 512 \\
    Multi-modal & yes & yes & yes & yes & yes & yes & no \\
    Prior dist. & Logistic & clip $\mathcal N(0,I)$ & $\mathcal N(0,2.5^2I)$ & clip $\mathcal N(0,I)$ & $\mathcal N(0,I)$ & $\mathcal N(0,I)$ & $\mathcal N(0,I)$ \\
    Iteration & 1 & 2 & 16 & 2 & 5 & 16 & 1 \\ 
    \bottomrule
  \end{tabular}
  }
  \vspace{2pt}
  \begin{flushleft}
    \footnotesize \textsuperscript{\dag} \myalgo uses width 256 only for DMC Dog and width 512 otherwise for controlling the parameter count.
  \end{flushleft}
\end{table}

\section{Extended Experimental Results and Ablations}
\label{app:more_results}

Unless stated otherwise, return tables report final IQM $\pm$ standard deviation across eight training seeds. We compute the IQM across seeds at each evaluation checkpoint and average the IQMs of the final four checkpoints. The standard deviation is computed across the eight per-seed returns, each averaged over those same checkpoints. Average entries are unweighted means of task-level IQMs.

\subsection{Benchmark Results and Sensitivity}
\label{app:benchmark_results_sensitivity}

\begingroup
\setlength{\textfloatsep}{6pt plus 1pt minus 1pt}
\setlength{\floatsep}{6pt plus 1pt minus 1pt}
\setlength{\intextsep}{6pt plus 1pt minus 1pt}
\setlength{\abovecaptionskip}{3pt}
\setlength{\belowcaptionskip}{0pt}

Figures~\ref{fig:appendix_iqm_learning_curves} and~\ref{fig:appendix_final_iqm_by_task} report learning curves and final IQM returns across all $12$ benchmark tasks. Figure~\ref{fig:appendix_compute_return_tradeoff} compares per-task returns with actor inference latency. Figure~\ref{fig:appendix_wallclock_iqm_curves} shows training wall-clock curves on Humanoid Run, Dog Run and H1 Hurdle using the median elapsed time across seeds at each evaluation step. Figures~\ref{fig:appendix_actor_width_sensitivity}, \ref{fig:appendix_depth_sensitivity} and~\ref{fig:appendix_rho_sensitivity} examine sensitivity to actor width, bridge depth and control budget.

\begin{figure}[!htbp]
  \centering
  \includegraphics[width=0.98\linewidth]{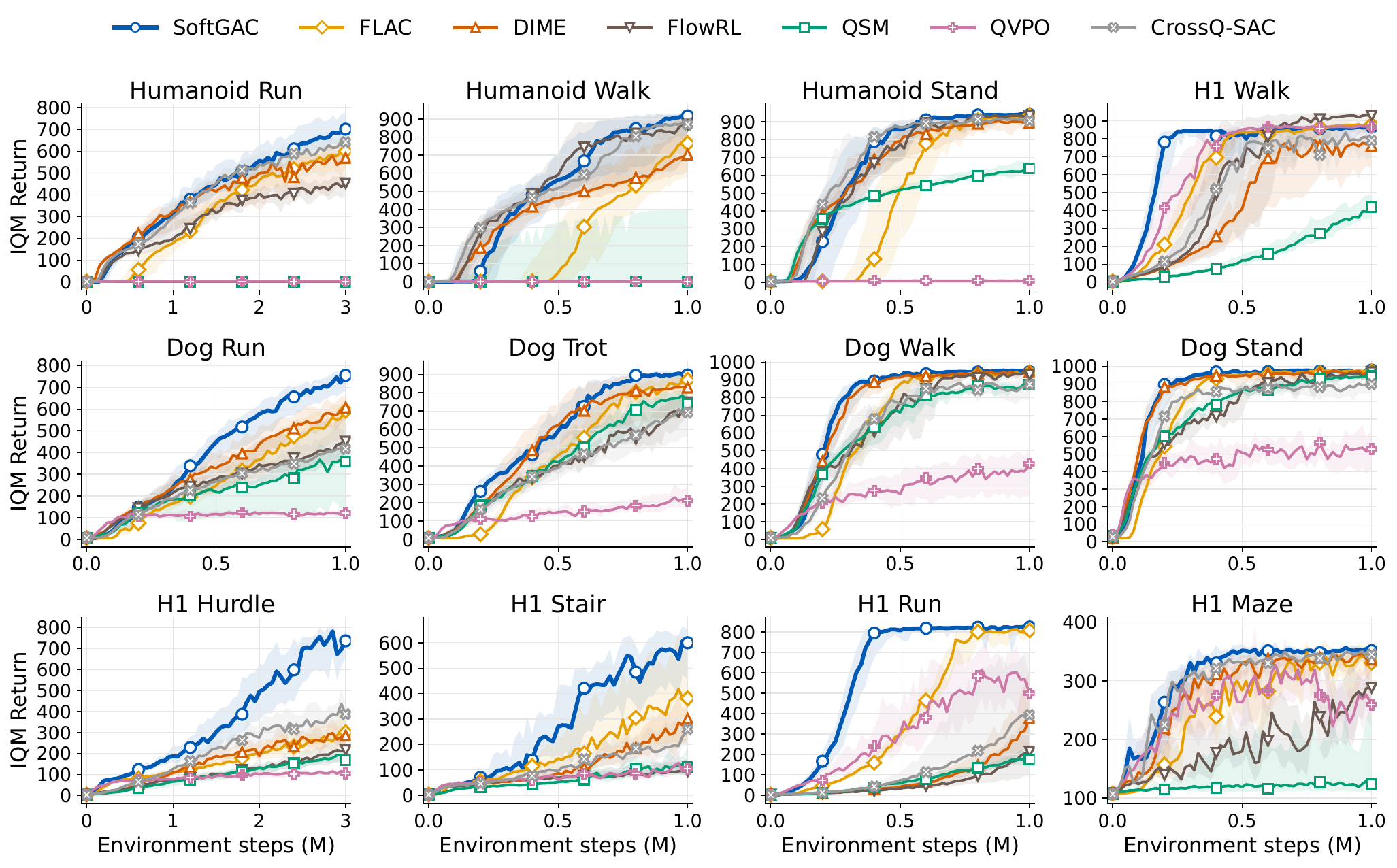}
  \caption{Full IQM learning curves on all benchmark tasks.}
  \label{fig:appendix_iqm_learning_curves}
\end{figure}

\begin{figure}[!htbp]
  \centering
  \includegraphics[width=0.98\linewidth]{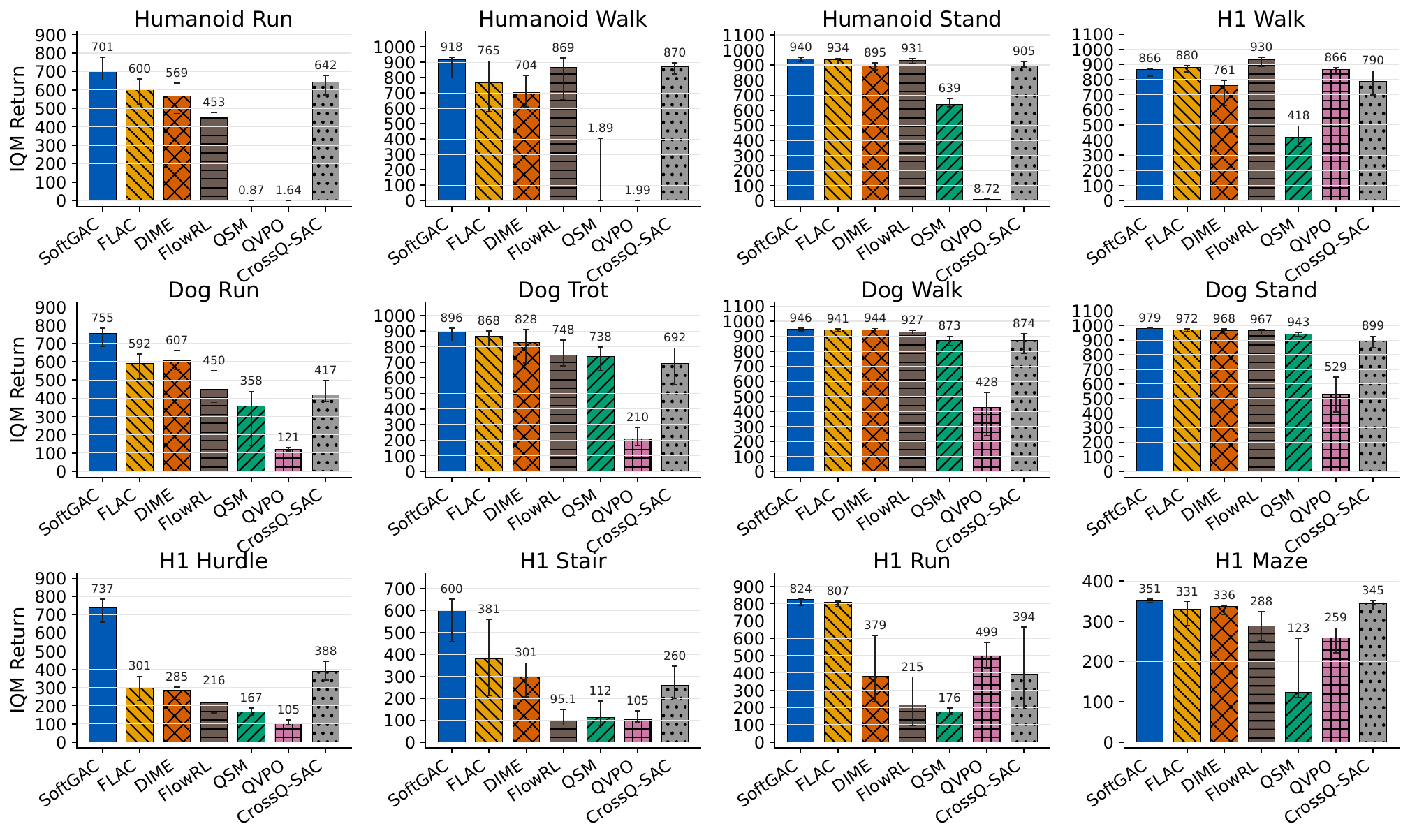}
  \caption{Final IQM return by task.}
  \label{fig:appendix_final_iqm_by_task}
\end{figure}

\myalgo achieves substantial gains on Humanoid Run, Dog Run, H1 Hurdle, H1 Stair and H1 Run, which involve high-dimensional locomotion and long-horizon credit assignment. On easier standing and walking tasks, several algorithms reach similarly high final returns, with larger differences in sample efficiency. The learning curves show when these differences emerge, while the final IQMs summarize performance over the last four evaluations.

\begin{figure}[!htbp]
  \centering
  \includegraphics[width=0.98\linewidth]{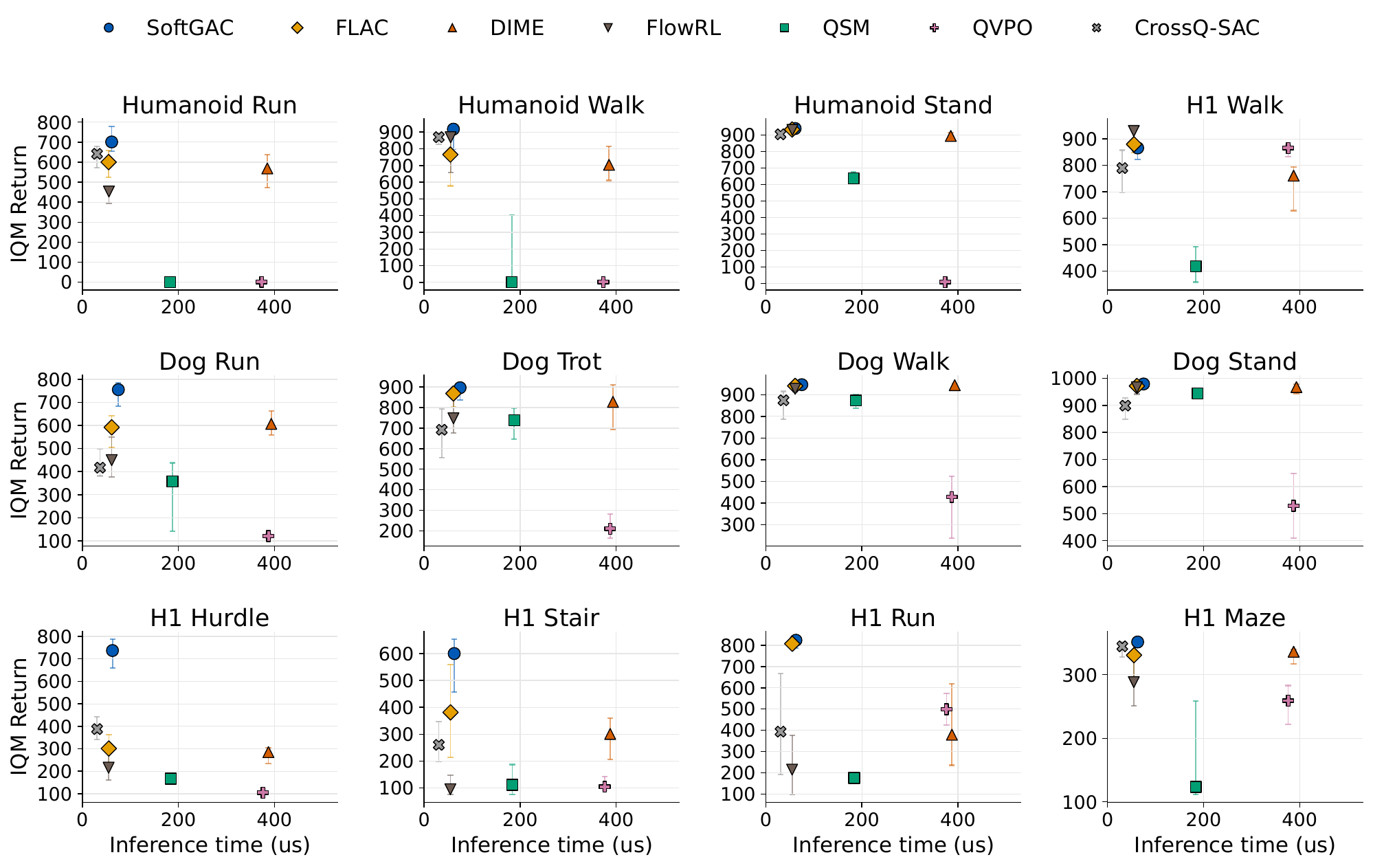}
  \caption{Compute-return tradeoff by task. The horizontal axis is per-action actor inference time, and the vertical axis is IQM return.}
  \label{fig:appendix_compute_return_tradeoff}
\end{figure}

The compute-return plots make the Pareto structure explicit. CrossQ-SAC has the cheapest Gaussian actor but lacks the multimodal generative policy class. DIME, QSM and QVPO sit in a higher-latency region because they evaluate iterative diffusion-style actors. \myalgo occupies a favorable region across the hard tasks because its actor remains close to one-pass inference cost while achieving higher or competitive IQM return.

\begin{figure}[!htbp]
  \centering
  \includegraphics[width=0.98\linewidth]{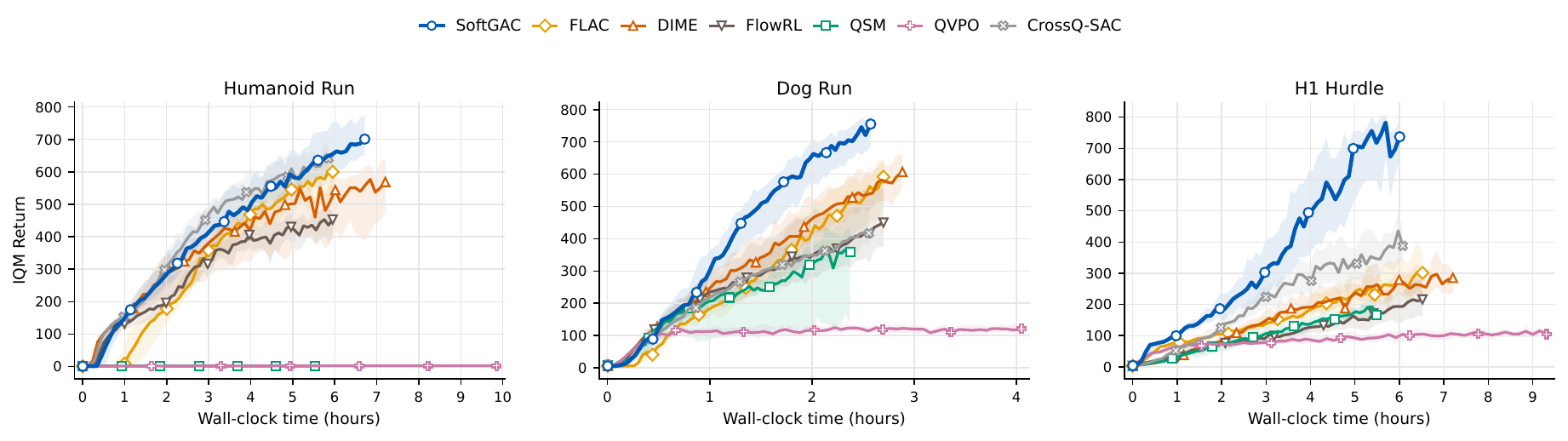}
  \caption{Wall-clock IQM learning curves on Humanoid Run, Dog Run and H1 Hurdle. The horizontal axis uses median seed wall-clock time at each evaluation step. All runs use the v6e-8 TPU VM, but the actual time depends on system scheduling and load.}
  \label{fig:appendix_wallclock_iqm_curves}
\end{figure}

Training wall-clock time includes environment interaction and depends on host load, system scheduling and run placement. We therefore report the median elapsed time across seeds at each environment step, alongside the isolated actor-latency measurements. \myalgo reaches high returns within a wall-clock budget comparable to low-NFE flow baselines, while DIME and QVPO incur the cost of repeated diffusion-network evaluations.

\begin{figure}[!htbp]
  \centering
  \includegraphics[width=0.92\linewidth]{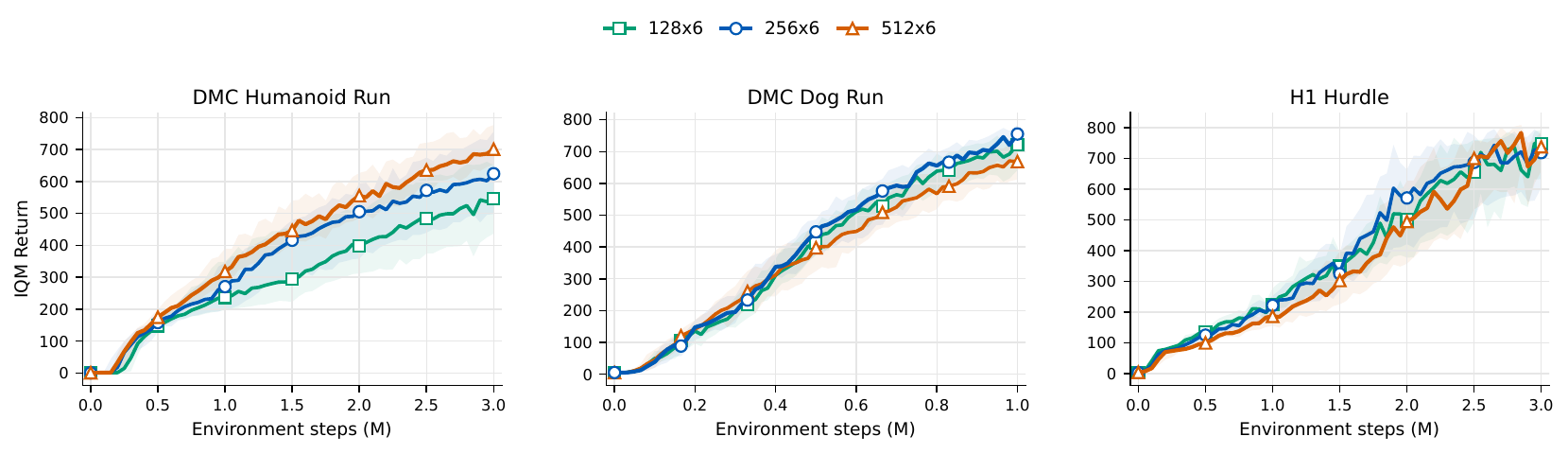}
  \caption{Actor-width sensitivity on Humanoid Run, Dog Run and H1 Hurdle.}
  \label{fig:appendix_actor_width_sensitivity}
\end{figure}

Actor-width sensitivity varies across tasks. Increasing width generally improves or stabilizes learning on Humanoid Run and H1 Hurdle. On Dog Run, the $128$-wide bridge already performs well with fewer parameters, while wider actors do not provide a consistent gain. We use width $512$ outside the Dog domain and width $256$ for Dog to keep actor parameter counts comparable to the baselines.

\begin{figure}[!htbp]
  \centering
  \includegraphics[width=0.92\linewidth]{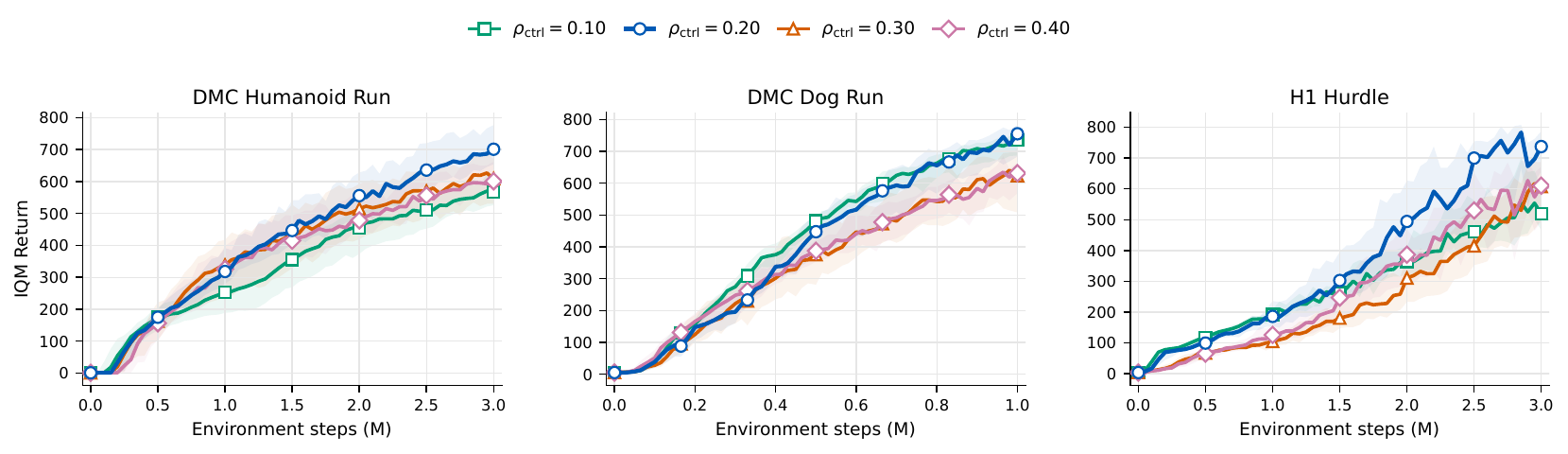}
  \caption{Target control-budget sensitivity on Humanoid Run, Dog Run and H1 Hurdle.}
  \label{fig:appendix_rho_sensitivity}
\end{figure}

The control-budget sweep reveals a tradeoff between reference matching and value-guided concentration. Small $\rho_{\mathrm{ctrl}}$ keeps the bridge close to the high-entropy reference and can restrict concentration on high-value actions. Larger budgets permit stronger control but do not always improve returns. The default $\rho_{\mathrm{ctrl}}=0.20$ gives stable performance across these three domains without task-specific tuning, consistent with the scale analysis in Appendix~\ref{app:rho_ctrl_choice}.

\begin{figure}[!htbp]
  \centering
  \includegraphics[width=0.52\linewidth]{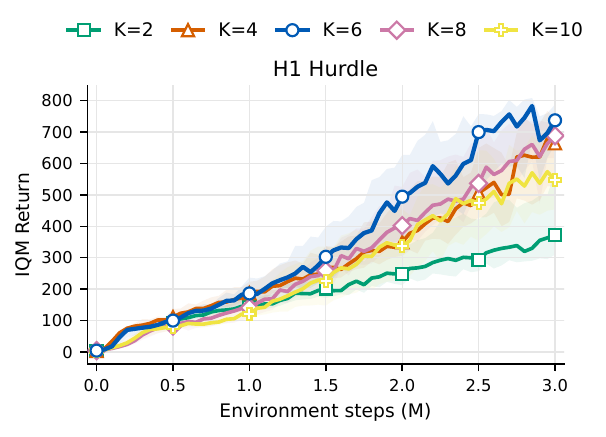}
  \caption{Bridge-depth sensitivity on H1 Hurdle with $K\in\{2,4,6,8,10\}$.}
  \label{fig:appendix_depth_sensitivity}
\end{figure}

Figure~\ref{fig:appendix_depth_sensitivity} varies the number of local transition blocks under the default schedule $\mathcal C_{\mathrm{target}}=0.2Kd_a$, so the sweep measures joint sensitivity to depth and total control budget. Under this schedule, $K=6$ gives the best H1 Hurdle result. The ablation in Appendix~\ref{app:variance_depth_ablation} varies depth at a fixed total budget and reports the realized path KL. Together, the two experiments support $K=6$ as a robust default, without a monotonic relationship between depth and return.

\subsection{Conditional Action Multimodality}
\label{app:multimodality_test}

We test whether the learned state-conditional action laws exhibit multimodality using projection pursuit with Hartigan's Dip statistic~\citep{KrauseLiebscher05,Kolyvakis25}. For each task-actor pair, we select the highest-return seed and collect $1000$ rollout states. At every state, we sample $8192$ actions and split them into $4096$ discovery samples and $4096$ held-out validation samples. The discovery split searches coordinate axes, fixed random directions reused across states and projection-pursuit directions. We then apply the Dip test only to the selected one-dimensional projection on the held-out split. A state is counted as multimodal when $p<0.05$.

\begin{table}[H]
  \centering
  \footnotesize
  \setlength{\tabcolsep}{5pt}
  \caption{State-conditional action multimodality. IQM returns aggregate eight training seeds. Multimodality percentages are computed over $1000$ rollout states from the highest-return seed.}
  \label{tab:multimodality_test}
  \begin{tabular}{lclcc}
    \toprule
    Task & $d_a$ & Actor & IQM return & Multimodal states \\
    \midrule
    H1 Hurdle & 19 & \myalgo & 737 & $53.8\%$ \\
              &    & CrossQ-SAC & 388 & $0.0\%$ \\
    H1 Run    & 19 & \myalgo & 824 & $18.6\%$ \\
              &    & CrossQ-SAC & 394 & $0.0\%$ \\
    Walker Walk & 6 & \myalgo & 980 & $0.2\%$ \\
                &   & CrossQ-SAC & 981 & $0.0\%$ \\
    \bottomrule
  \end{tabular}
\end{table}

The test detects substantial projected multimodality for \myalgo on both H1 tasks and none for CrossQ-SAC. On Walker Walk, where both methods achieve near-saturated returns, \myalgo rarely produces a significantly multimodal projection. The observed multimodality thus varies with the task rather than appearing uniformly across the bridge policies.

\subsection{Path Regularizer and Base Distribution}
\label{app:regularizer_base_ablation}

To isolate the choice of regularizer, we replace the path KL with the GMM endpoint entropy estimator used by DACER. The estimator fits a three-component GMM to $200$ sampled endpoint actions, and we use DACER's target-entropy setting. The actor architecture, critic and training budget remain fixed. In a separate ablation, we replace the logistic base with an isotropic standard Gaussian. The full path KL then additionally contains the fixed initial-law term $D_{\mathrm{KL}}(p_0^{\mathrm{Gaussian}}\|r_0^{\mathrm{logistic}})$.

\begin{table}[H]
  \centering
  \footnotesize
  \setlength{\tabcolsep}{7pt}
  \caption{Path-regularizer and base-distribution ablations. Per-task values are IQM $\pm$ standard deviation over eight seeds.}
  \label{tab:regularizer_base_ablation}
  \begin{tabular}{lcc}
    \toprule
    \multicolumn{3}{l}{\textit{(a) Analytical path KL versus GMM endpoint entropy}} \\
    Task & Path-KL regularizer & GMM entropy estimator \\
    \midrule
    Dog Run (1M) & $755\pm64$ & $586\pm77$ \\
    H1 Hurdle (3M) & $737\pm107$ & $347\pm152$ \\
    H1 Run (1M) & $824\pm32$ & $192\pm57$ \\
    H1 Stair (1M) & $600\pm125$ & $270\pm155$ \\
    Average & 729 & 349 \\
    \midrule
    \multicolumn{3}{l}{\textit{(b) Logistic versus standard-Gaussian base}} \\
    Task & Logistic base & Gaussian base \\
    \midrule
    H1 Hurdle (3M) & $737\pm107$ & $670\pm175$ \\
    H1 Run (1M) & $824\pm32$ & $823\pm14$ \\
    H1 Stair (1M) & $600\pm125$ & $562\pm136$ \\
    Average & 720 & 685 \\
    \bottomrule
  \end{tabular}
\end{table}

The path-KL regularizer yields higher IQM returns on all four tasks. The GMM approach replaces the analytical path regularizer with a finite-sample endpoint-density approximation and adds the cost of action sampling and mixture fitting. Differences between the two base distributions are smaller. The logistic base gives slightly higher average return, but the differences are small relative to run-to-run variability, indicating robustness to these two base choices. We use the logistic base because its tanh image is uniform, which simplifies the reference construction.

\subsection{Diffusion and Normalizing-Flow Baselines}
\label{app:recent_baselines}

We compare \myalgo with DACER, MaxEntDP and SDAC under the same C51+CrossQ main critic, preserving each baseline's actor objective, entropy treatment and auxiliary components. Actor parameter counts are approximately matched, and all three diffusion baselines use a shared denoiser with $20$ sampling steps. We evaluate all four methods on Walker Run, H1 Run and H1 Hurdle.

\begin{table}[H]
  \centering
  \footnotesize
  \setlength{\tabcolsep}{6pt}
  \caption{Comparison with recent diffusion-policy baselines. Values are IQM $\pm$ standard deviation over eight seeds.}
  \label{tab:recent_diffusion_baselines}
  \begin{tabular}{lcccc}
    \toprule
    Actor & Walker Run (1M) & H1 Run (1M) & H1 Hurdle (3M) & Average \\
    \midrule
    \myalgo & $823\pm21$ & $824\pm32$ & $737\pm107$ & 795 \\
    DACER & $860\pm32$ & $537\pm238$ & $377\pm75$ & 591 \\
    MaxEntDP & $629\pm58$ & $23\pm8$ & $18\pm3$ & 223 \\
    SDAC & $836\pm34$ & $36\pm5$ & $60\pm24$ & 311 \\
    \bottomrule
  \end{tabular}
\end{table}

All four methods learn on Walker Run, but their returns differ markedly on the harder H1 tasks. MaxEntDP uses Monte Carlo approximations in both its actor update and entropy-aware critic target, while SDAC trains a diffusion actor with reweighted score matching. \myalgo avoids endpoint-density estimation and backpropagation through a repeated denoiser, and achieves the highest average IQM across these three tasks.

We also adapt a conditional normalizing-flow policy~\citep{NormalizingFlowsCC25} to online SAC. It retains exact squashed-action log density and its flow architecture, while sharing the same main critic and training budget as the other methods.

\begin{table}[H]
  \centering
  \footnotesize
  \setlength{\tabcolsep}{4pt}
  \caption{Comparison with DACER and a conditional normalizing-flow actor. Values are IQM $\pm$ standard deviation over eight seeds.}
  \label{tab:normalizing_flow_baseline}
  \resizebox{\linewidth}{!}{%
  \begin{tabular}{lcccccc}
    \toprule
    Actor & Dog Run (1M) & H1 Hurdle (3M) & H1 Run (1M) & H1 Stair (1M) & Humanoid Run (3M) & Average \\
    \midrule
    \myalgo & $755\pm64$ & $737\pm107$ & $824\pm32$ & $600\pm125$ & $701\pm64$ & 724 \\
    DACER & $542\pm76$ & $377\pm75$ & $537\pm238$ & $164\pm102$ & $541\pm82$ & 432 \\
    Normalizing Flow & $503\pm95$ & $325\pm79$ & $590\pm234$ & $376\pm131$ & $597\pm73$ & 478 \\
    \bottomrule
  \end{tabular}%
  }
\end{table}

A normalizing flow makes endpoint density tractable by requiring every transformation to be bijective with a tractable Jacobian determinant. This dimension-preserving structure constrains its architecture. The stochastic bridge blocks in \myalgo need neither invertibility nor tractable Jacobians because local transition densities already expose the analytical path KL. \myalgo is strongest on every task in Table~\ref{tab:normalizing_flow_baseline}.

\subsection{Actor Cost and Training Wall-Clock Time}
\label{app:additional_compute}

Table~\ref{tab:additional_actor_cost} reports evaluation-time actor parameters and single-action latency on an Apple M3 Pro CPU, excluding compilation. During evaluation, \myalgo follows its deterministic mean path and bypasses variance heads, DACER omits GMM fitting and the normalizing flow omits log-determinant computation.

\begin{table}[H]
  \centering
  \footnotesize
  \setlength{\tabcolsep}{5pt}
  \caption{Evaluation-only actor parameters and inference latency by domain.}
  \label{tab:additional_actor_cost}
  \begin{tabular}{lclll}
    \toprule
    Actor & NFE & \shortstack{DMC Dog\\params / time} & \shortstack{HumanoidBench\\params / time} & \shortstack{DMC Humanoid\\params / time} \\
    \midrule
    \myalgo & 1 & $0.526$M / $75\,\mu$s & $0.342$M / $63\,\mu$s & $0.410$M / $61\,\mu$s \\
    DACER & 20 & $0.426$M / $414\,\mu$s & $0.318$M / $400\,\mu$s & $0.328$M / $404\,\mu$s \\
    Normalizing Flow & 1 & $0.413$M / $129\,\mu$s & $0.344$M / $97\,\mu$s & $0.351$M / $98\,\mu$s \\
    \bottomrule
  \end{tabular}
\end{table}

\begin{table}[H]
  \centering
  \footnotesize
  \setlength{\tabcolsep}{7pt}
  \caption{Actor cost and median training wall-clock time on H1 Hurdle at $3$M environment steps. Actor parameters count the complete trainable actor. Training uses one TPU core per seed.}
  \label{tab:hurdle_wallclock_baselines}
  \begin{tabular}{lcccc}
    \toprule
    Actor & NFE & Actor params & Inference & Training time \\
    \midrule
    \myalgo & 1 & $0.342$M & $63\,\mu$s & $6.0$ h \\
    DACER & 20 & $0.318$M & $400\,\mu$s & $7.6$ h \\
    MaxEntDP & 20 & $0.399$M & $509\,\mu$s & $24.9$ h \\
    SDAC & 20 & $0.318$M & $410\,\mu$s & $9.3$ h \\
    \bottomrule
  \end{tabular}
\end{table}

DACER is $5.5$--$6.6\times$ slower than \myalgo at deployment across the three domains, while the one-pass normalizing flow is $1.5$--$1.7\times$ slower. On H1 Hurdle, \myalgo also trains faster than all three diffusion baselines. MaxEntDP is especially costly because its actor update and entropy approximation use extensive Monte Carlo sampling.

\subsection{Transfer to an XQC Critic Backend}
\label{app:xqc_backend}

To test sensitivity to the critic architecture and update rule, we train \myalgo, DIME and SAC with the same XQC critic~\citep{xqc26}. Unlike the C51+CrossQ configuration, XQC uses an exponential-moving-average target critic. We keep the actor objectives, actor architectures and training budgets unchanged.

\begin{table}[H]
  \centering
  \footnotesize
  \setlength{\tabcolsep}{7pt}
  \caption{Actor comparison after replacing the shared critic backend with XQC. Values are IQM $\pm$ standard deviation over eight seeds.}
  \label{tab:xqc_actor_matrix}
  \begin{tabular}{lcccc}
    \toprule
    Actor & Dog Run (1M) & H1 Slide (1M) & H1 Stair (1M) & Average \\
    \midrule
    \myalgo & $666\pm64$ & $829\pm160$ & $447\pm124$ & 647 \\
    DIME & $483\pm45$ & $663\pm115$ & $298\pm101$ & 482 \\
    SAC & $486\pm47$ & $478\pm121$ & $288\pm115$ & 417 \\
    \bottomrule
  \end{tabular}
\end{table}

\myalgo achieves the highest IQM on all three tasks with the XQC critic. Its average IQM is $34\%$ higher than DIME and $55\%$ higher than SAC. The advantage therefore persists beyond the C51+CrossQ critic used in the other comparisons.

\subsection{Learned Variance and Fixed-Budget Bridge Depth}
\label{app:variance_depth_ablation}

The local Gaussian path KL contains both drift and variance terms. We isolate them on H1 Stair by fixing the transition variance $\sigma^2$ throughout training. The setting $\sigma^2=1$ matches the reference variance, makes the variance KL zero and leaves only the drift cost.

\begin{table}[H]
  \centering
  \footnotesize
  \setlength{\tabcolsep}{9pt}
  \caption{Transition-variance ablation on H1 Stair at $1$M steps. Values are IQM $\pm$ standard deviation over eight seeds.}
  \label{tab:variance_ablation}
  \begin{tabular}{lc}
    \toprule
    Variance during training & H1 Stair IQM \\
    \midrule
    Learned variance (default) & $600\pm125$ \\
    Fixed $\sigma^2=2.0$ & $42\pm4$ \\
    Fixed $\sigma^2=1.0$ (drift-only KL) & $529\pm126$ \\
    Fixed $\sigma^2=0.5$ & $243\pm172$ \\
    Fixed $\sigma^2=0.1$ & $8\pm0$ \\
    \bottomrule
  \end{tabular}
\end{table}

The learned variance is strongest, while the drift-only variant remains effective. Fixed variances above or below the reference value degrade learning, which supports adaptive transition-scale exploration during training. The variance heads are bypassed during deterministic evaluation.

To isolate bridge depth from the target KL budget, we evaluate Dog Run with $\mathcal C_{\mathrm{target}}=0.2\times6\times d_a=45.6$ for every $K$, where $d_a=38$. Equivalently, $\rho_{\mathrm{ctrl}}=1.2/K$. The hidden width remains $256$.

\begin{table}[H]
  \centering
  \footnotesize
  \setlength{\tabcolsep}{8pt}
  \caption{Bridge depth on Dog Run under a fixed total control budget. Return values report IQM $\pm$ standard deviation across eight seeds.}
  \label{tab:fixed_budget_depth}
  \begin{tabular}{cccc}
    \toprule
    Bridge depth $K$ & Actor params & Empirical path KL & Dog Run IQM \\
    \midrule
    4 & $351$K & $45.6\pm0.1$ & $746\pm62$ \\
    5 & $438$K & $45.7\pm0.1$ & $716\pm56$ \\
    6 (default) & $526$K & $45.6\pm0.1$ & $755\pm64$ \\
    7 & $613$K & $45.6\pm0.1$ & $704\pm99$ \\
    8 & $701$K & $45.6\pm0.2$ & $711\pm78$ \\
    \bottomrule
  \end{tabular}
\end{table}

The realized KL remains at the common target and performance has no monotonic trend over $K=4$ to $8$. The method is therefore insensitive to moderate depth changes under a fixed budget on this task. Scaling to deeper bridge networks may require architecture-specific optimization, which we leave for future work.

\endgroup

\makeatletter
\if@preprint
\else
  \clearpage
  \input{checklist.tex}
\fi
\makeatother

\end{document}

%% file: checklist.tex
\section*{NeurIPS Paper Checklist}

\begin{enumerate}

\item {\bf Claims}
    \item[] Question: Do the main claims made in the abstract and introduction accurately reflect the paper's contributions and scope?
    \item[] Answer: \answerYes{}
    \item[] Justification: The abstract and introduction state the proposed path-space soft generative actor-critic method, the low-cost bridge-policy design and the empirical compute-return claims. The scope is explicitly limited to state-based simulated online continuous control, and the paper includes matching theoretical analysis, finite-reference discussion, ablations and experiments.

\item {\bf Limitations}
    \item[] Question: Does the paper discuss the limitations of the work performed by the authors?
    \item[] Answer: \answerYes{}
    \item[] Justification: The final section discusses the finite-step Gaussian bridge, reference and base-law choices, added implementation complexity, bridge hyperparameters and the current restriction to state-based simulated control. It identifies image-conditioned learning, world models and real-robot evaluation as future work. Appendices~\ref{app:finite_reference_bias} and~\ref{app:more_results} quantify the relevant biases and robustness.

\item {\bf Theory assumptions and proofs}
    \item[] Question: For each theoretical result, does the paper provide the full set of assumptions and a complete (and correct) proof?
    \item[] Answer: \answerYes{}
    \item[] Justification: The main text states the assumptions for each proposition, theorem, corollary and lemma, and Appendix~\ref{app:proofs} restates each result before giving the proof. Appendix~\ref{app:finite_reference_bias} separately analyzes the finite-step reference endpoint bias used by the implementation.

    \item {\bf Experimental result reproducibility}
    \item[] Question: Does the paper fully disclose all the information needed to reproduce the main experimental results of the paper to the extent that it affects the main claims and/or conclusions of the paper (regardless of whether the code and data are provided or not)?
    \item[] Answer: \answerYes{}
    \item[] Justification: Appendix~\ref{app:experimental_setup} provides benchmark tasks, training hardware, implementation details, the control-energy budget choice, hyperparameters, parameter counts and optimizer settings. Appendix~\ref{app:pseudocode} gives pseudocode for the training loop. Appendix~\ref{app:more_results} states the protocols for the multimodality, baseline, critic-backend and ablation experiments.

\item {\bf Open access to data and code}
    \item[] Question: Does the paper provide open access to the data and code, with sufficient instructions to faithfully reproduce the main experimental results, as described in supplemental material?
    \item[] Answer: \answerYes{}
    \item[] Justification: The implementation and experiment scripts are available in the repository linked below the abstract. Appendix~\ref{app:experimental_setup} describes the environments, implementation details, hyperparameters and compute resources. The experiments use standard simulator benchmarks rather than a new dataset.

\item {\bf Experimental setting/details}
    \item[] Question: Does the paper specify all the training and test details (e.g., data splits, hyperparameters, how they were chosen, type of optimizer) necessary to understand the results?
    \item[] Answer: \answerYes{}
    \item[] Justification: Appendix~\ref{app:env_task_domains} reports the main environments, observation dimensions and action dimensions. Appendices~\ref{app:training_hardware}, \ref{app:implementation_details}, \ref{app:rho_ctrl_choice} and~\ref{app:hyperparameters} report hardware, implementation details, the control-energy budget heuristic, optimizer settings, main hyperparameters, actor and critic sizes and parameter counts. Appendix~\ref{app:more_results} gives protocols and settings for the extended evaluations.

\item {\bf Experiment statistical significance}
    \item[] Question: Does the paper report error bars suitably and correctly defined or other appropriate information about the statistical significance of the experiments?
    \item[] Answer: \answerYes{}
    \item[] Justification: The main experiments report IQM over $8$ random seeds with $95\%$ bootstrap confidence intervals. These intervals are shown in the main figures and Appendix~\ref{app:more_results}. Extended return tables in the same appendix report IQM $\pm$ standard deviation over eight seeds and identify the one-seed multimodality analysis separately.

\item {\bf Experiments compute resources}
    \item[] Question: For each experiment, does the paper provide sufficient information on the computer resources (type of compute workers, memory, time of execution) needed to reproduce the experiments?
    \item[] Answer: \answerYes{}
    \item[] Justification: Appendix~\ref{app:training_hardware} reports the cloud TPU type, CPU model, CPU count, memory and operating system of the training machines. Appendices~\ref{app:more_results} and~\ref{app:additional_compute} report representative wall-clock curves, actor latency and training time for difficult tasks.
    
\item {\bf Code of ethics}
    \item[] Question: Does the research conducted in the paper conform, in every respect, with the NeurIPS Code of Ethics \url{https://neurips.cc/public/EthicsGuidelines}?
    \item[] Answer: \answerYes{}
    \item[] Justification: The work studies an algorithmic reinforcement-learning method in simulated control benchmarks. It does not involve human subjects, private data, web scraping, safety-critical deployment or personally identifiable information.

\item {\bf Broader impacts}
    \item[] Question: Does the paper discuss both potential positive societal impacts and negative societal impacts of the work performed?
    \item[] Answer: \answerNA{}
    \item[] Justification: This is foundational algorithmic work evaluated only in simulated continuous-control benchmarks. We do not identify a direct societal deployment path or direct societal impact beyond general advances in reinforcement learning.
    
\item {\bf Safeguards}
    \item[] Question: Does the paper describe safeguards that have been put in place for responsible release of data or models that have a high risk for misuse (e.g., pre-trained language models, image generators, or scraped datasets)?
    \item[] Answer: \answerNA{}
    \item[] Justification: The paper does not release high-risk pretrained models, image generators, scraped datasets or user-facing systems. The experiments use simulated control environments.

\item {\bf Licenses for existing assets}
    \item[] Question: Are the creators or original owners of assets (e.g., code, data, models), used in the paper, properly credited and are the license and terms of use explicitly mentioned and properly respected?
    \item[] Answer: \answerYes{}
    \item[] Justification: Appendix~\ref{app:env_task_domains} lists the benchmark assets used for evaluation, gives URLs for DeepMind Control Suite and HumanoidBench, and states their licenses. Table~\ref{tab:baseline_code_references} lists the official baseline repositories used as implementation references, and the baseline methods are credited through their cited papers.

\item {\bf New assets}
    \item[] Question: Are new assets introduced in the paper well documented and is the documentation provided alongside the assets?
    \item[] Answer: \answerYes{}
    \item[] Justification: The accompanying repository provides implementations of the proposed method and baselines. The repository contains experiment scripts, configuration files and instructions needed to run the reported methods. The paper does not introduce a new dataset, benchmark suite or pretrained model asset.

\item {\bf Crowdsourcing and research with human subjects}
    \item[] Question: For crowdsourcing experiments and research with human subjects, does the paper include the full text of instructions given to participants and screenshots, if applicable, as well as details about compensation (if any)? 
    \item[] Answer: \answerNA{}
    \item[] Justification: The work does not involve crowdsourcing or research with human subjects.

\item {\bf Institutional review board (IRB) approvals or equivalent for research with human subjects}
    \item[] Question: Does the paper describe potential risks incurred by study participants, whether such risks were disclosed to the subjects, and whether Institutional Review Board (IRB) approvals (or an equivalent approval/review based on the requirements of your country or institution) were obtained?
    \item[] Answer: \answerNA{}
    \item[] Justification: The work does not involve crowdsourcing or research with human subjects, so IRB approval is not applicable.

\item {\bf Declaration of LLM usage}
    \item[] Question: Does the paper describe the usage of LLMs if it is an important, original, or non-standard component of the core methods in this research? Note that if the LLM is used only for writing, editing, or formatting purposes and does \emph{not} impact the core methodology, scientific rigor, or originality of the research, declaration is not required.
    \item[] Answer: \answerNA{}
    \item[] Justification: LLMs are not part of the proposed method, experiments, baselines, theoretical analysis or evaluation protocol.

\end{enumerate}